\documentclass[accepted]{uai2026} 

\usepackage[american]{babel}

\usepackage{natbib} 
\usepackage{mathtools} 
\usepackage{siunitx} 
\usepackage{booktabs} 
\usepackage{tikz} 

\usepackage[table]{xcolor}

\usepackage{amsthm} 
\newtheorem{theorem}{Theorem}[section]
\newtheorem{proposition}[theorem]{Proposition}

\usepackage{amsmath,amsfonts,bm}

\def\eqref#1{equation~\ref{#1}}

\def\1{\bm{1}}
\newcommand{\data}{\mathcal{D}}

\def\vzero{{\bm{0}}}

\def\vmu{{\bm{\mu}}}

\def\vb{{\bm{b}}}

\def\vd{{\bm{d}}}

\def\vx{{\bm{x}}}

\def\vz{{\bm{z}}}

\def\mA{{\bm{A}}}

\def\mI{{\bm{I}}}

\def\mW{{\bm{W}}}
\def\mX{{\bm{X}}}

\def\mSigma{{\bm{\Sigma}}}

\DeclareMathAlphabet{\mathsfit}{\encodingdefault}{\sfdefault}{m}{sl}
\SetMathAlphabet{\mathsfit}{bold}{\encodingdefault}{\sfdefault}{bx}{n}

\newcommand{\ie}{i.e., }

\newcommand{\eg}{e.g., }

\newcommand{\Real}{\mathbb{R}}
\newcommand{\Natural}{\mathbb{N}} 

\newcommand{\bigO}[1]{\mathcal{O}\left( #1 \right)}

\newcommand{\N}[1]{\mathcal{N}\left( #1\right)}

\newcommand{\p}[1]{p{\left( #1 \right)}}
\newcommand{\cp}[2]{p{\left( #1 \mid #2 \right)}}

\newcommand{\todoi}[1]{\todo[inline]{#1}}

\newcommand{\db}[1]{\textcolor{blue}{#1}}

\newcommand{\pF}{\mathrm{F}}
\newcommand{\F}[1]{\pF{\left( #1 \right)}}
\newcommand{\cF}[2]{\pF{\left( #1 \mid #2 \right)}}
\newcommand{\eF}[3]{\pF_{#1}^{#2}{\left( #3 \right)}}

\newcommand{\pwl}[2]{p_{#1}{\left( #2 \right)}}

\newcommand{\wvzt}{\widetilde{\vz_t}}
\newcommand{\mut}{\boldsymbol{\mu_t}}
\newcommand{\cmut}[1]{\boldsymbol{\mu_{#1}}}
\newcommand{\Sigmat}{\boldsymbol{\Sigma_t}}

\usepackage{hyperref}
\usepackage[capitalise]{cleveref}
\usepackage{url}
\usepackage{enumitem}
\newlist{Properties}{enumerate}{2}
\setlist[Properties]{label=Property \arabic*., font=\textbf, itemindent=*}

\usepackage{caption}
\usepackage{subcaption}  
\usepackage{wrapfig} 
\usepackage{tikz}
\usepackage[most]{tcolorbox}
\usepackage{multirow}

\usepackage{algorithm}
\usepackage{algpseudocode} 

\usepackage[textsize=tiny]{todonotes}

\title{Anomaly detection in time-series via inductive biases  in the latent space of conditional normalizing flows}

\author[1]{\href{mailto:david.baumgartner@ntnu.no}{David Baumgartner}{}}
\author[2,3,4]{\href{mailto:eliezer.silva@uc.pt}{Eliezer de Souza da Silva}{}}
\author[2,5]{\href{mailto:iurteaga@bcamath.org}{Iñigo Urteaga}{}}
\affil[1]{%
    Department of Computer Science\\
    Norwegian University of Science and Technology\\
    Norway
}
\affil[2]{%
  BCAM --- Basque Center for Applied Mathematics\\
  Spain
}
\affil[3]{%
    Department of Informatics Engineering\\
    University of Coimbra\\
    Portugal
}
\affil[4]{%
  CISUC/LASI --- Centre for Informatics and Systems\\
  University of Coimbra\\
  Portugal
}
\affil[5]{%
  IKERBASQUE -- Basque Foundation for Science\\
  Spain
}

\begin{document}
\maketitle

\begin{abstract}
Deep generative models for anomaly detection in multivariate time-series are typically trained by maximizing observed data likelihood. However, likelihood in observation space measures marginal density rather than conformity to structured temporal dynamics, and therefore can assign high probability to anomalous or out-of-distribution samples.
We address this structural limitation by relocating the notion of anomaly to a prescribed latent space. 
We introduce explicit inductive biases in conditional normalizing flows, modeling time-series observations within a discrete-time state-space framework that constrains latent representations to evolve according to prescribed temporal dynamics. 
Under this formulation, expected behavior corresponds to compliance with a specified distribution over latent trajectories, while anomalies are defined as violations of these dynamics.
Anomaly detection is consequently reformulated as a statistically grounded compliance test, such that observations are mapped to latent space and evaluated via goodness-of-fit tests against the prescribed latent evolution. 
This yields a principled decision rule that remains effective even in regions of high observation likelihood.
Experiments on synthetic and real-world time-series demonstrate reliable detection of anomalies in frequency, amplitude, and observation noise, while providing interpretable diagnostics of model compliance.
\end{abstract}
\section{Introduction}
\label{sec:intro}

Anomaly detection (AD), also known as outlier or novelty detection, is a fundamental problem in signal processing and machine learning with applications ranging from fraud detection and predictive maintenance to medical imaging and network security~\citep{chandola_anomaly_2009,pimentel_review_2014}.
Anomaly detection for time-series have temporal dependence that introduces additional complexity, requiring methods that distinguish structured deviations from expected dynamics~\citep{schmidl_anomaly_2022}.

Classical AD typically relies on hand crafted anomaly scores, such as reconstruction error, distance to cluster centers, or density estimates,
which must be converted into binary decisions via manually chosen thresholds.
In practice, these thresholds are often tuned using labeled examples of anomalous behavior or extensive domain expertise, which can be costly, brittle across domains, and difficult to justify statistically when anomalies are inherently rare and heterogeneous~\citep{chandola_anomaly_2009}.

In many realistic settings, annotated anomalies are scarce or unavailable, whereas large amounts of data describing \emph{expected behavior} can be collected at low cost.
This motivates unsupervised anomaly detection, where the learning algorithm is trained on data representing expected behavior and must infer deviations without explicit supervision~\citep{pimentel_review_2014}.
In this setup, the objective is to learn an inherent notion of what constitutes expected behavior directly from data, so that deviations from this learned notion can be flagged as anomalies.

Formally, we posit a parametric generative model with density $\pwl{\theta}{\vx}$ encoding the \emph{expected behavior} within a region of the data space with a well defined statistical distribution $\pwl{\theta}{\cdot}$, with samples outside of that region regarded as out-of-distribution (OOD) or anomalous.
With this probabilistic framework, anomalies can be scored via quantities intrinsic to the generative model, such as likelihoods, reconstruction-based discrepancies, or hybrid criteria that combine data and latent-space representations. Crucially, this also allows us to develop an approach for AD using goodness-of-fit (GOF) statistical tests.

Modern deep generative models (DPMs) provide flexible parametric families for approximating complex data distributions and have become popular building blocks for probabilistic modeling ~\citep{goodfellow_generative_2014,kingma_auto-encoding_2022}. Recent studies demonstrate the effectiveness of and interest in unsupervised DPMs for AD across domains~\citep{pang_deep_2021,xia_gan-based_2022,berahmand_autoencoders_2024}. 

However, these models are not free of pitfalls:
recent work on OOD detection has revealed that deep generative models can behave counterintuitively~\citep{serra_input_2019,li_position_2025},
sometimes assigning higher likelihoods to OOD inputs than to in-distribution data. 
Notably, \citet{morningstar_density_2021} cautioned us against the naive use of negative log-likelihood (NLL) as a proxy for \emph{typicality}, 
highlighting the importance of appropriate inductive biases and model diagnostics for AD~\citep{zhang_understanding_2021}.
These shortcomings reveal a need for generative models that
not only fit the nominal data distribution well,
but also induce a robust and interpretable separation between expected and unexpected patterns.
As argued by \citet{li_position_2025}, likelihood-based training alone does not provide an inherent notion of unexpected behavior, and such distinctions must arise from explicit structural inductive biases encoded in the model.


In this work, we propose an unsupervised, state-space probabilistic framework for time-series data and study how inductive biases in representation-space can overcome some of these limitations.

Specifically, we define a normalizing-flow-based DPM with explicit inductive biases
to capture the expected behavior of data via latent, prescribed dynamics.
Conditioned on a successful training procedure,
the framework provides an inherent (statistical) definition of expected (in-distribution)
versus anomalous (out-of-distribution) behavior, enabling unsupervised anomaly detection.

Our contributions include:
\begin{itemize}
    \item A state-space deep generative model that couples a conditional normalizing flow with  explicit (\eg linear-Gaussian) latent dynamics, constraining observations to map to a temporally coherent, latent trajectories of prescribed density.
    \item A statistically principled, unsupervised (\ie label- and threshold-free) anomaly detector based on a goodness-of-fit (\eg a multivariate Kolmogorov-Smirnov) test in latent space, capable of identifying anomalies even in high-density regions of the DPM.
    \item A built-in, compliance diagnostic to identify when the DPM training procedure is successful in enforcing the prescribed inductive bias, explicitly signaling when the unsupervised AD procedure is ready for testing.
    \item Showcasing empirically
    ($i$) the in-distribution AD issues with negative log-likelihood-based scores,
    ($ii$) the robustness and detection accuracy of the proposed framework,
    and ($iii$) its successful performance in-par with established baselines on real-data AD scenarios.
\end{itemize}




\section{Background and Related Work}
\label{sec:background}

We provide below a succinct background on AD and normalizing flows, 
situating our work within the existing literature.



Anomaly, outlier, and OOD detection all aim to identify patterns in data that deviate from expected behavior~\citep{ruff_unifying_2021,chandola_anomaly_2009}. 
Deep generative models have been widely explored
across this and other domains over the last few years~\citep{pang_deep_2021,xia_gan-based_2022,berahmand_autoencoders_2024}.
Specifically, time-series anomaly detection (TSAD) deals with connected and ordered data points over time, which yields additional AD challenges.
\citet{schmidl_anomaly_2022} provide a good overview of the complexities inherent in TSAD and the existing approaches to addressing it.
\citet{sorbo_navigating_2024} further highlight metrics used in this domain and point out issues that make fair comparisons across methods difficult.
VAE based DPM methods for AD often use reconstruction probability or related likelihood surrogates as anomaly scores, enabling a more principled interpretation than ad hoc error thresholds~\citep{xu_unsupervised_2018}.
Similarly, GAN-based approaches such as TadGAN~\citep{geiger_tadgan_2020} or USAD~\citep{audibert_usad_2020} employ sequence-to-sequence generators and critics to reconstruct time-series,
deriving anomaly scores from reconstruction residuals and critic responses,
highlighting anomalous segments rather than isolated point outliers.


Normalizing flows (NF) are generative probabilistic models that learn a bijective transformation that maps the data to a known distribution~\citep{papamakarios_normalizing_2021}.
These models enable both model fitting and sample generation; \ie
NFs can efficiently compute the log-likelihood of observations ($\log \p{\vx}$) and generate new data ($\vx \sim  \eF{}{-1}{\N{\cdot}}$), which requires tractable Jacobian determinants.
We denote the normalizing and likelihood of a data point $\vx$ as $\p{\vz = \F{\vx}}$ and the generation of a data point $\vx = \eF{}{-1}{\vz \sim \N{\cdot}}$, where $\F{\cdot; \theta}$ is the NF with parameters $\theta$ which we omit when it is clear from the context.

To model multivariate time-series data, 
we employ conditional normalizing flows (CNF)~\citep{rasul_multivariate_2020}.
NFs can be conditioned in different ways;
in this case, we focus on transformation-layer conditioning.
This setting requires paired inputs for the CNF, \ie $\cF{\vx}{\vd}: \Real^D \times \Real^M \rightarrow \Real^D$. 
Notice that $\vd \in \Real^M$ can be have different dimension than $\vx$.
Unsupervised training of CNF is performed using the data's negative log-likelihood, according to a base distribution and the change-of-variable formula,
with loss function $\mathcal{L}(\vx, \vd) = \log \p{\cF{\vx}{\vd}} + \sum_{l=1}^{L} \log \vert \det \mathbf{J}(\pF_l)(\vx_{l}|\vd) \vert,$
where $L$ is the total number of transformation layers in the CNF.
The Jacobian, $\mathbf{J}_l$, of each layer $l$ is a function of the input $\vx_{l}$ and the condition $\vd$.
\cite{kang_traffic_2022,moon_multivariate_2023,guan_conditional_2023,chen_robust_2025} use this conditioning concept for TSAD,
all focusing on various levels of embedding creation of the temporal context for the definition of $\vd$. 

\section{Proposed Methodology}
\label{sec:method}

In this section, we present a time-series probabilistic framework suited for anomaly detection,
based on a discrete-time state-space model
that captures the temporal evolution of the data
according to prescribed inductive biases.
In line with \citet{li_position_2025},
we argue that structural inductive biases are required for a model to know what constitutes \textit{unexpected behavior} solely from maximizing likelihood on samples with \textit{expected behavior}. 

We propose not only to
($i$) shift the AD mechanism from the ambient to the semantic latent space,
but also to ($ii$) enforce inductive biases in such latent space,
to capture the \textit{expected behavior manifold}.
For the former, we use CNFs
for the observation- to latent-space mapping.
For the latter, we prescribe expected dynamics
over the latent representations' evolution over time.
Rather than modeling static distributions over data,
we constrain the dynamics of the learned representations
to evolve according to prescribed inductive biases.
Consequently, compliance with the inductive bias,
\ie the prescribed expected dynamics, serves as the definition of
\textit{expected behavior}:
\begin{itemize}
    \item \textbf{Training-time compliance:} 
    With sufficient model capacity and under model likelihood optimization,
    the imposed inductive bias shapes the latent trajectories of observations,
    defining which temporal evolutions within representation space constitute \textit{expected behavior}.
    
    \item \textbf{Testing-time compliance:}
    At inference time, the mapped representations of new observations should be
    similar to latent trajectories observed during training and
    consistent with the prescribed latent distribution. 
    Hence, an input sequence, after the non-linear observation- to latent-space mapping, 
    is an anomaly if it does not comply with such \textit{expected behavior}. 
\end{itemize}

With this framework, we cast unsupervised time-series anomaly detection as a check for inductive bias compliance, rather than relying on model-based likelihood scores.

We proceed to describe in detail each of the components of the framework:
the CNF-based state-space model with inductive bias in Section~\ref{ssec:model}, along with its training procedure,
and the time-series anomaly detection procedure as a check for inductive bias compliance in Section~\ref{ssec:unsupervised_AD}.

\subsection{The probabilistic state-space model}
\label{ssec:model}

We denote the sequence of observations with $\vx_t \in \Real^D$, indexed by time $t \in \Natural$, and introduce a corresponding latent state variable $\vz_t \in \Real^D$ to represent the underlying semantic dynamics of the system.

There are two key components to the proposed deep probabilistic model:
\begin{enumerate}
    \item \textbf{A Conditional Normalizing Flow (CNF)}, parameterized by $\theta$ and with temporal context window $\mW$, mapping each observation $\vx_t$ into a latent representation $\vz_t$, conditioned on a finite history of previous observations, $\mW_t = \vx_{t-k:t-1}$:
    \begin{align}
    \label{eq:nf}
    \vz_t &= \cF{\vx_t}{\mW_t; \theta} \sim \N{\mut, \boldsymbol{\Sigma_t}} \footnotemark.
    \end{align}
    \footnotetext{
    For any random variable $\vz_t \sim \N{\vmu_t, \mSigma_t}$ with known sufficient statistics $\vmu_t, \mSigma_t$
    we define its whitened counterpart $\wvzt = \mSigma^{-1/2}_t (\vz_t - \vmu_t) \sim \N{\vzero, \mI}$.
    }
    
    \item \textbf{The Latent Dynamics}, imposing an explicit inductive bias on the temporal evolution of the latent representations, via deterministic dynamics of its mean:
    \begin{align}
    \label{eq:inductive_bias_0}
    \vmu_{0} & \sim \p{\phi_0} \;, \\
    \label{eq:inductive_bias_t}
    \vmu_t & = \psi(\vmu_{t-1} ; \phi) \; .
    \end{align}

    The inductive bias,
    which constrains latent trajectories to follow a prescribed dynamical law,
    is defined in a generic form,
    as a deterministic dynamic map $\psi(\cdot | \phi)$.
\end{enumerate}

Various inductive biases can be used within this general framework,
depending on the specific application, prior problem knowledge, or other practical reasons.
Next, we describe a simple inductive bias that is useful both for explainability and for anomaly detection.
    
\subsubsection{The linear-Gaussian latent dynamical model (LG-LDM)}
\label{sssec:linear_latent_dynamics}

We can impose linear-Gaussian expected dynamics of the latent representation via
\begin{align}
\label{eq:inductive_bias_linear_Gaussian}
\vmu_{0} &\sim \p{\phi_0} = \N{\vzero, \mI} \; , \\
\vmu_t &= \psi(\vmu_{t-1} ; \phi) = \mA \vmu_{t-1} + \vb \; ,
\end{align}
with uncorrelated latent uncertainty, $\mSigma_t = \mI, \; \forall t$,
where the learnable parameters of the LG-LDM are $\phi = \{\mA, \vb\}$.

With such inductive bias,
we impose a closed-form time-evolution for the expected representation $\vmu_t$ of the form
\begin{align}
\label{eq:inductive_bias_linear_Gaussian_mu}
    \vmu_t &= \mA^t \vmu_0 + \sum_{k=0}^{t-1} \mA^k \vb,
\end{align}
where $\mA^k$ is the $k$-th matrix power and the sum is a matrix–vector product at each time $k$.

If $\vmu_0 = \vzero$ and $\rho(\mA) < 1$, then $\lim_{t \rightarrow \infty} \vmu_t = (\mI-\mA)^{-1} \vb$, by the matrix geometric series:
\ie the trajectory converges to a fixed point, and if $\vb=\vzero$, then such a fixed point is the origin.

\subsubsection{Probabilistic training of the state-space model}
\label{sssec:training_procedure}

The training procedure that aligns the CNF mapping of observations with the prescribed inductive bias, \ie the latent dynamics, 
follows minimization of the negative log-likelihood (NLL) of the model over training data.

We jointly train the CNF parameters $\theta$
and the parameters of the latent dynamics $\phi$
via NLL optimization.
The procedure ensures that
the learned latent trajectories are consistent with the deterministic evolution encoded by $\psi(\cdot)$, both in their dynamics and its prescribed distribution.

Notably, the inductive bias-based DPM can be optimized either sequentially over the complete time-series or using sub-sequences over-time, with computational benefits.
We present these alternatives below, with the choice depending on the computational resources of practitioners.

\paragraph{Training with complete time-series.}
We can sequentially train the CNF over the complete time-series via
\begin{align}
\label{eq:cnf_loss_sequenctia}
\begin{split}
\mathcal{L}_{seq}(\data) &= - \frac{1}{T} \sum_{t=1}^{T} \biggr[ \log \pwl{t}{\cF{\vx_t}{\mW_t}} \\
& \hspace{10ex}+
\sum_{l=1}^{L} \log \vert \det \mathbf{J}(\pF_l)(\vx_{t,l}|\mW_t) \vert \biggr],
\end{split} \hspace{-2ex}\\
\text{where } &
    \pwl{t}{\cdot} = \N{\mut, \Sigmat} \;, \\
& \text{with }
    \cmut{0} \sim \p{\phi_0} \;, 
    \mut = \psi(\cmut{t-1} ; \phi) \; .
\end{align} 
Note that this training approach scales with $T$, as the optimization step of the loss depends on the full time-series.

\paragraph{Training with mini-batches of observations over time.}
Alternatively, for Markovian dynamics as in \cref{eq:inductive_bias_t}, a more computationally efficient approach is possible,
where the NLL is computed for sub-sequences over time.

We can split a time-series over $B$ sub-sequences/mini-batches of size $BS$, 
\ie $\mX_{b} = \{\vx_{bs}\}_{bs=1}^{BS}$,
and write
\begin{align}
\label{eq:cnf_loss_mini_batch}
\begin{split}
\mathcal{L}_{batch}(\data) &= - \frac{1}{B} \sum_{b=1}^{B} \frac{1}{BS} \biggr[ \log \pwl{b}{\cF{\mX_{b}}{\mW_{b}}} \\
    & \hspace{9ex} +
\sum_{l=1}^{L} \log \vert \det \mathbf{J}(\pF_l)(\mX_{b,l}|\mW_{b}) \vert \biggr],
\end{split} \hspace{-4ex} \\
\text{with } & \pwl{b}{\cdot} = \{\N{\vmu_{bs}, \mSigma_{bs}}\}_{bs=1}^{BS} \;,
\end{align}
where $\mW_{b}$ are the corresponding temporal contexts to each sample in the mini-batch,
and $\vmu_0 \sim \N{\vzero, \mI}$ in mini-batch $1$, and the last mean $\vmu_{bs}$ from the previous mini-batch $b-1$ is the initializing $\vmu_0$ for each consecutive mini-batch $b$.

In this mini-batched version of training,
we can update model parameters after each sub-sequence,
and $\mathcal{L}_{batch}(\data)$ reports the average loss over all mini-batches.

This alternative is viable whenever we can pre-compute all the necessary
expected representations for each mini-batch of the input time-series,
as we can then split the observations of a time-series over time
into sub-sequences/mini-batches,
reducing the computational dependency on the time-series length $T$.
Obviously, when $BS = T$ and $B = 1$, we recover the full sequential version. 

\subsection{Latent dynamic compliance as a statistical test}
\label{ssec:unsupervised_AD}

In our probabilistic framework,
the prescribed inductive bias is reflected in the learned latent dynamics.
Therefore, compliance of observations with such inductive bias and unsupervised anomaly detection are two-sides of the same coin:
\textit{expected and unexpected behavior} are nothing but
\textit{compliance and non-compliance} with the inductive bias.

Because the inductive bias in Equations~\ref{eq:inductive_bias_0}-~\ref{eq:inductive_bias_t} enforces
distributional constraints on latent trajectories,
learned representation compliance can be measured via a goodness-of-fit (GOF) 
test between the prescribed evolving distribution and the empirical distribution of the mapped trajectories.
The GOF indicates, at training time, whether the model learned the prescribed dynamics;
at test time, it is used to detect anomalies.

In what follows,
we use the multivariate Kolmogorov-Smirnov (MV-KS) GOF test to quantify compliance,
due to its generality, its non-parametric nature, and its exactness even under small sample sizes~\citep{naaman_tight_2021}.

The MV-KS test quantifies GOF to a prescribed distribution, providing the KS statistic ($s$) and a critical value ($\tau$) for decision-making:
we can reject that 
samples come from the prescribed distribution if the KS statistic is bigger than the critical value, \ie $s \geq \tau$.
The MV-KS critical value is non-parametric and dependent on the sample size.

\paragraph{The MV-KS test for CNF-learned latent representations.}
Let $\widehat F_\theta$ denote a trained CNF according to prescribed latent dynamics with fitted parameters $\widehat\phi$.
We define the CNF-based latent as
$\vz_t = \widehat F_\theta(\vx_t \mid \mW_t)$ and
$\wvzt = \widehat\Sigma_t^{-1/2}(\vz_t - \widehat \vmu_t)$
as its whitened counterpart.
Conditional on a well-fitted model,
whitening yields independent standard normal variables,
\ie $\wvzt \stackrel{\mathrm{i.i.d.}}{\sim} \N{\vzero, \mI}$.
Because the MV-KS test is a test of a fully specified continuous distribution against i.i.d. samples, its null distribution is universal (distribution-free) and invariant under the bijective transformation used to generate the samples:
\ie the CNF nonlinearity does not invalidate the MV-KS test
---details and the proof are provided in~\cref{axsec:ks_validity}.
The utility of the representation space MV-KS test to assess inductive bias compliance
is twofold, both as a training diagnostic and as an anomaly detector.

\paragraph{Compliance to inductive bias: successful learning.}
We expect and require that a well-trained CNF model maps observations of \textit{expected behavior} into latent dynamics that comply with the inductive bias.
Therefore, the result of a GOF test indicates whether the trained model is compliant with the prescribed dynamics.
We define the training or fitting diagnostic $\mathrm{FIT}$ based on the KS test statistic of the (whitened) training sequence
$s_T = \mathrm{KS}(\widetilde{\vz}_{t=1}^T, \N{\vzero,\mI})$
as $\mathrm{FIT} = \mathbf{1}_{s_T \leq \tau_T}$,
where the critical value $\tau_T$ is dependent on the sequence length.
Due to its non-parametric definition and sample-size adaptability, the MV-KS statistic computed on latent training trajectories provides
a data-driven notion of expected behavior,
\ie it measures how compliant latent representations are,
bypassing the need for threshold tuning.
$\mathrm{FIT}$ is
an automatic and trustworthy model training diagnostic metric,
as good model learning must enforce latent representations to obey prescribed latent distribution dynamics ---see an illustrative example in \cref{fig:fit_vs_nofit_4d},
where for the LG-LDM, the inductive bias prescribes standard Gaussian representations, whitened over time, to which latent trajectories comply.

\paragraph{Non-compliance to inductive bias: anomaly detection.}
On the contrary, we can identify \textit{unexpected behavior} by checking whether an input sequence, after the non-linear observation- to latent- transformation by the CNF, follows or not the expected dynamics, as determined by a GOF test.

Specifically, for the MV-KS test, we can flag a set of observations of length $w$ as an anomaly according to
\begin{equation}
\operatorname{H}(s_w) =
\begin{cases}
    \text{expected behavior,} & \text{if } s_w < \tau_w \\
    \text{unexpected behavior,} & \text{otherwise}.
\end{cases}
\end{equation}
Notice how, within this unsupervised AD framework based on MV-KS tests,
no manual AD threshold needs to be defined.
Namely, the CNF trained according to available data determines the operating point for anomaly detection.


\section{Evaluation}
\label{sec:eval}

We assess the proposed framework for anomaly detection on synthetic and real-world data by enforcing linear Gaussian dynamics in the latent space, \ie the LG-LDM inductive bias in Section~\ref{sssec:linear_latent_dynamics}. 
We first describe our experimental setup, including the hyperparameter optimization strategy and the reported metrics. 
We then study synthetic data to demonstrate the limitations of standard likelihood-based AD and illustrate the benefits of the proposed framework.
Finally, we compare our method against baselines on real-world data to assess its sensitivity and performance.

\subsection{Experimental Setup}
\label{ssec:setup}

Our framework is optimized using the NLL presented in \cref{sssec:training_procedure}, which jointly optimizes the CNF and, in these experiments, the LG-LDM parameters.
For the CNF, we use a variant of the RealNVP architecture, originally proposed by \citet{dinh_density_2017} and extended by \citet{rasul_multivariate_2020}. 

We evaluate our framework on synthetic and real data.
In the synthetic data experiments, we assess the performance across a range of frequency, amplitude, and noise changes as sequence anomalies (\cref{fig:data_space_test}).
On the real-world data, we evaluate and compare performance on a subset of univariate\footnote{
Univariate series are duplicated to adapt to the default multivariate interface of our framework.
} and multivariate real-world sequences (TSB-AD)~\citep{liu2024}.
We select a fair set of statistics and deep learning methods from TSB-AD as baselines, based on their rankings, as reported in~\citep{liu2024}.
Although we calculate all metrics based on TSB-AD's implementation, we can only compare results in detail on the \textit{Volume Under the Surface Precision-Recall} (VUS-PR) curve. 
This is because per-sequence results are publicly available only for this metric.
We also report aggregated, oracle-thresholded results on NLL \& KS for the proposed AD in \Cref{axssec:univarite_sequence_results} and \cref{axssec:multivariate_sequence_results}.

For model hyperparameter tuning,
we use a grid search for synthetic data
and Optuna~\citep{akiba_optuna_2019} (TPE; 20 warm-up + 100 trials) for TSB-AD data.
The hyperparameters and their evaluated ranges are listed in~\cref{tab:hyperparameters}.
Code, models, and tables are available in a GitHub repository~\footnote{\url{https://github.com/iurteagalab/probML_anomaly_detection}}.

\begin{table}[]
    \centering
    
    \caption{Overview of the hyperparameters and their search ranges. The hyperparameters refer to different parts of our framework, such as the LG-LDM (D), the CNF (C), or the implementation (S). ``NA'' stands for sequential training.}
    \label{tab:hyperparameters}
    
    \begin{tabular}{llll}
    & Parameter & Synthetic & Real-world \\
    \midrule
    D & Learning $\vb$ & [True, False] & True \\
    \midrule
    C & Temporal Context & [20, 40, 100] & \numrange{5}{500} \\
    C & CNF layers & [6, 8, 12] & \numrange{3}{40} \\
    C & Hidden layers & [1, 3] & \numrange{1}{20} \\
    C & Hidden layer size & [64, 128] & \numrange{4}{200} \\
    \midrule
    S & Batch size & [NA, 2048] & 2048 \\
    \end{tabular}
\end{table}


\subsection{Experimental Results}
\label{ssec:exp-results}


\paragraph{Negative log-likelihood Vs inductive bias compliance.}
We present in \cref{fig:main_figure} an example time-series
with amplitude and frequency anomalies blended in
---see shaded background areas and original training data with dashed lines.
\Cref{tab:synthetic_res_table} contains the metric results for this experiment,
comparing AD performance based on the raw NLL score and the proposed MV-KS alternative.

\begin{figure}[h!]
    \centering
    \includegraphics[width=1\linewidth]{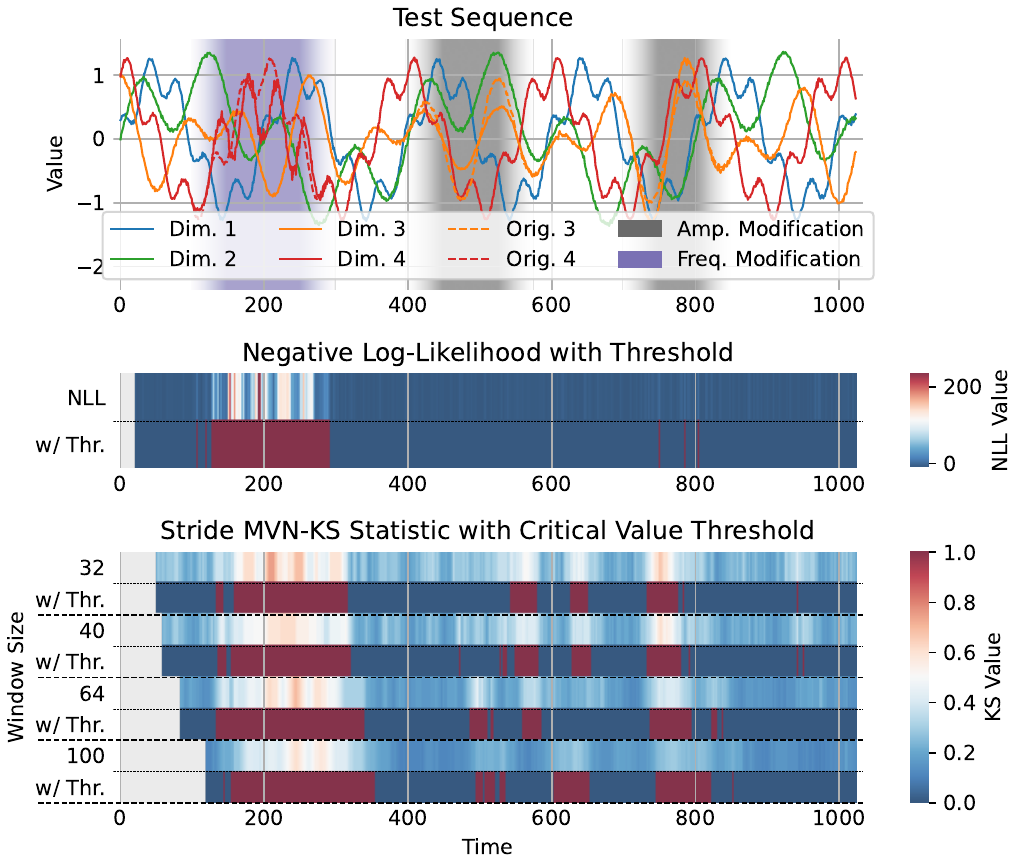}
    \caption{
    Anomaly detection comparison between the proposed model when using NLL scores and the proposed GOF test.
    (Top) The input sequence, contaminated with amplitude and frequency modifications (in shaded regions).
    (Middle) Heatmap of the continuous and thresholded NLL (threshold set to the maximum NLL value in training):
    NLL-based scores fail to detect amplitude changes.
    (Bottom) The MV-KS score heatmap, showing successful detection of the modified sections. 
    As detailed in \Cref{tab:synthetic_res_table}, our approach outperforms the NLL baseline across various metrics, highlighted by a $7\%$ improvement in the affiliation F1 score.
    }
    \label{fig:main_figure}
\end{figure}

Results in \cref{fig:main_figure} highlight the limitations of NLL-based scoring,
as it is only able to detect the first anomaly around $t\approx200$.
Frequency-based modifications within the sequence are flagged by both scores,
but for amplitude-based anomalies,
only the MV-KS-based AD identifies the anomalies around $t\approx500$ and $t\approx800$.
This occurs because the CNF tends to place these observed data points
into high-density latent regions.
However, even in these high-likelihood representation space regions,
these trajectories do not obey \textit{expected behavior}.
Hence, as we test $\wvzt$ compliance via the MV-KS GOF,
its score ---as shown by the heatmaps in \cref{fig:main_figure}---
reveals unexpected behavior around all three anomaly sequences around $t \approx \{200, 500, 800\}$.

\begin{table}[t!]
    \centering
    \caption{
    Non-parametric MV-KS results for the test case in \cref{fig:main_figure}.
    Parametric results covering NLL and MV-KS w/wo critical value (CV) thresholding or the AUC-ROC curve as threshold for scoring, are in~\cref{axsec:synthetic_f1_res_table}.
    }
    \label{tab:synthetic_res_table}
\begin{tabular}{lrrrr}
& \multicolumn{2}{c}{AUC} & \multicolumn{2}{c}{VUS} \\
Score Source & PR & ROC & PR & ROC \\
\midrule
NLL & 78.8 & 70.9 & 92.8 & 90.9 \\
KS w=20 & 79.9 & 71.9 & 92.4 & 90.4 \\
KS w=32 & 79.1 & 72.8 & 93.8 & 92.0 \\
KS w=40 & 80.1 & 73.0 & 94.6 & 92.8 \\
KS w=64 & \textbf{82.1} & \textbf{73.1} & \textbf{96.0} & \textbf{93.7} \\
KS w=100 & 74.5 & 66.3 & 92.7 & 91.1 \\
KS w=200 & 58.3 & 64.8 & 90.4 & 92.2 \\
KS w=400 & 78.3 & 69.1 & 89.5 & 86.9 \\
\end{tabular}
\end{table}

We further investigate the proposed unsupervised AD methodology's robustness
to varying anomaly intensities,
by assessing how sensitive the MV-KS test is
to variations in amplitude, frequency and observation noise
with respect to the training time-series.

\Cref{fig:data_space_test} summarizes the methodology's performance,
showcasing the MV-KS score-difference between the training sequence and the modified time-series, as well as the interpretable representation space for a subset of selected cases.

The latent, whitened $\wvzt$ space depicted in \cref{fig:data_space_test}  is compliant with
the LG-LDM inductive bias only on training data,
\ie Gaussian shape samples with $KS=0.02$;
while it showcases visually noticeable departures
(a doughnut shape, stretched and enlarged samples)
for minor and major test-set deviations.
Even sequences with a minor change in frequency and amplitude,
that result in a visually subtle difference,
are flagged by the MV-KS test: $s=0.076>0.02, \tau=0.075$.

\begin{figure}[!h]
    \centering
    \includegraphics[width=1\linewidth]{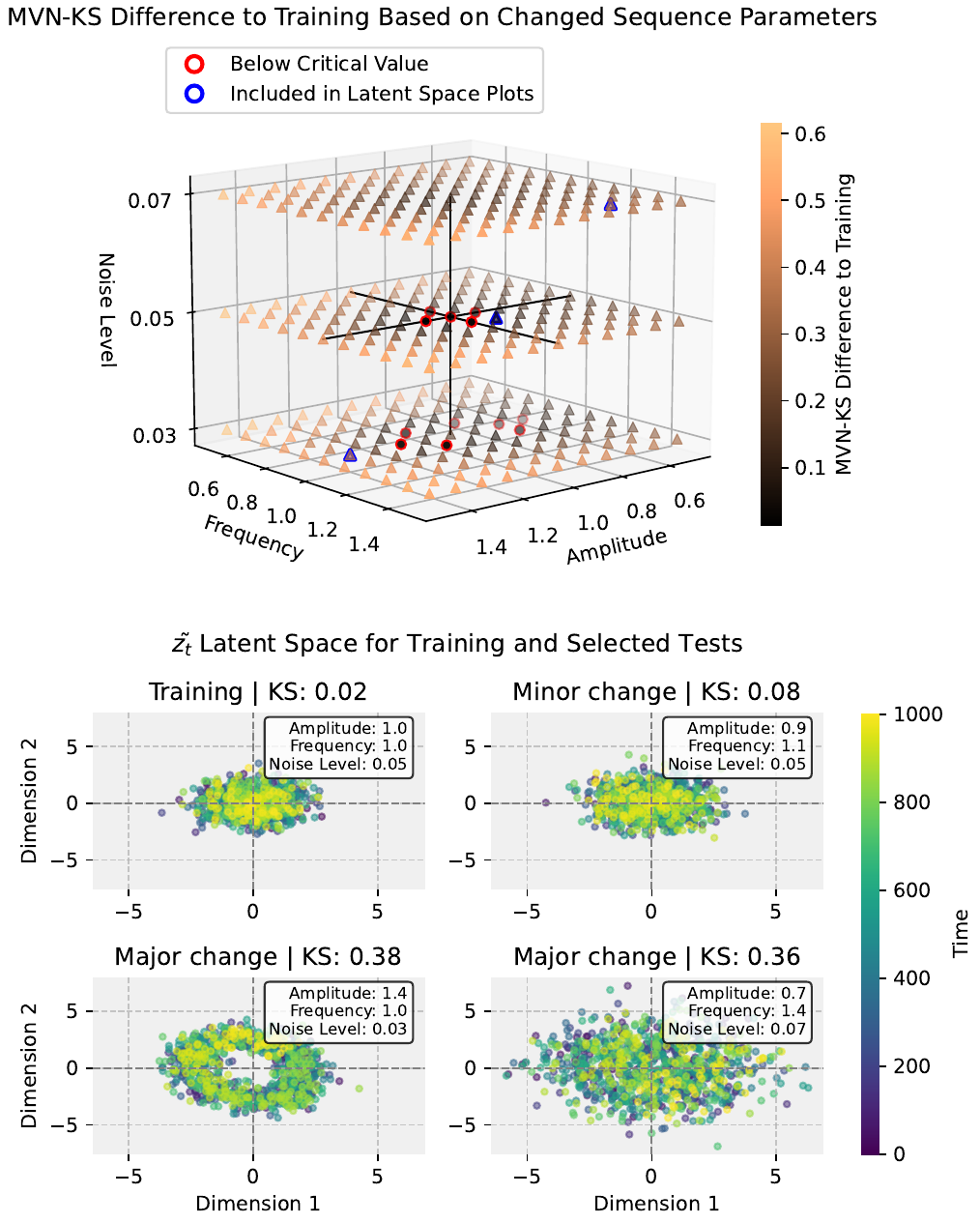}
    \caption{
    Sensitivity of the latent space GOF to input data variations.
    (3D Plot) The test range, where the color gradient encodes the relative difference in MVN-KS values compared to the baseline training sequence at the center. 
    Points with a {\color{red}red} border remain below the critical threshold, despite having higher MVN-KS values than the training baseline.
    (2x2 Scatter Plot) Visualizes the $\wvzt$ latent space for the training sequence and three modified test cases ({\color{blue}blue} borders). 
    Overall, the reported MVN-KS values indicate that greater deviations from the training sequences lead to more pronounced violations of the GOF criteria.
    }
    \label{fig:data_space_test}
\end{figure}

\paragraph{Anomaly Detection via MV-KS tests: the impact of window size.}
A key AD design choice is what window-size to use to determine whether a time-series subsequence is an anomaly or not. 
In our case, these window size is directly related to the statistical power of the GOF test, as it determines the length of the latent trajectory and, hence, the number of samples used in the test:
the MV-KS test depends on the number of samples available to compute its statistics.

As shown in \cref{fig:main_figure}, 
window sizes of $w \leq 20$ are usually too noisy to separate anomalous from compliant distributions in the evaluated 4-dimensional time-series,
while $w \geq 200$ dilutes the identified anomalies, degrading computed metrics.
We report corresponding AD metrics for different window sizes in \cref{tab:synthetic_res_table}.

In general, we recommend a window size of $\bigO{D^3}$,
which matches the empirical sweet spot of $w = 64$ observed in this experiment (see \cref{tab:synthetic_res_table}):
window size $w = 64$ achieves AUC-PR $82.1$ and VUS-PR $96.0$,
in comparison with NLL-based performane of $78.8$ and $92.8$, respectively.
Large gains are also observed in ranged-F1 ($65.5$ vs.\ $36.0$) and affiliation F1 ($92.6$ vs.\ $85.0$) scores.

\paragraph{Anomaly Detection via MV-KS tests: the impact of the LG-LDM inductive bias.}
We here report results for a real-data ablation study to assess the importance of the LG-LDM inductive bias,
as a minimal and interpretable choice.
We observe that the use of such inductive bias improves AD performance over a standard Gaussian $\N{\vzero, \mI}$ latent prior across diverse time-series benchmarks (UCR, TAO, MITDB, Exathlon, Yahoo, Opportunity).

Specifically, LG-LDM yields higher KS VUS-PR ($0.921 \pm 0.092$ vs. $0.836 \pm 0.213$) and slightly better KS VUS-ROC ($0.984 \pm 0.027$ vs. $0.977 \pm 0.032$).
Hence, we conclude that imposing explicit latent dynamics (even if linear and Gaussian) helps the proposed framework better capture temporal compliance, via MV-KS tests on the learned latent space.

\paragraph{Training diagnostics via inductive bias compliance.}
As discussed in Section~\ref{ssec:unsupervised_AD},
we can test inductive bias compliance, and hence good model learning,
by computing the FIT metric, \ie
by MV-KS testing training trajectories mapped onto $\wvzt$.
If for most training subsequences mapped onto $\wvzt$,
a MV-KS score below the critical value is attained,
we conclude that the inductive bias is well-prescribed.
We can therefore be confident that the model is well-trained and is ready for AD detection:
the MV-KS critical value is a valid, unsupervised threshold.

In \cref{fig:fit_vs_nofit_4d}, we present a comparison of 3 CNF variants,
one per column, where we can visually (and through MV-KS scores) assess compliance with inductive bias.
In the left column, we observe the latent state of a well-trained model,
with very low $KS=0.038$ score. 
In contrast, in the other two columns,
we observe how $\wvzt$ trajectories deviate systematically from the (in this case, Gaussian) inductive bias.
Due to excessive CNF capacity, mismatched context length, or data complexity,
training data are not compliant with the imposed dynamics,
invalidating the automatic KS-based threshold.
This training diagnostic,
where we can quantify the fraction of sequences passing the MV-KS test,
is useful for real-world data
---see the \textit{FIT} column in \cref{tab:u_vus-pr_aggregated_comparison}---
as it directly predicts when the trained model and its MV-KS score can be trusted for AD.
Nonetheless, we recall that $\mathrm{FIT}$ assesses statistical validity,
not AD performance.
\begin{figure}[!t]
    \centering
    \includegraphics[width=1\linewidth]{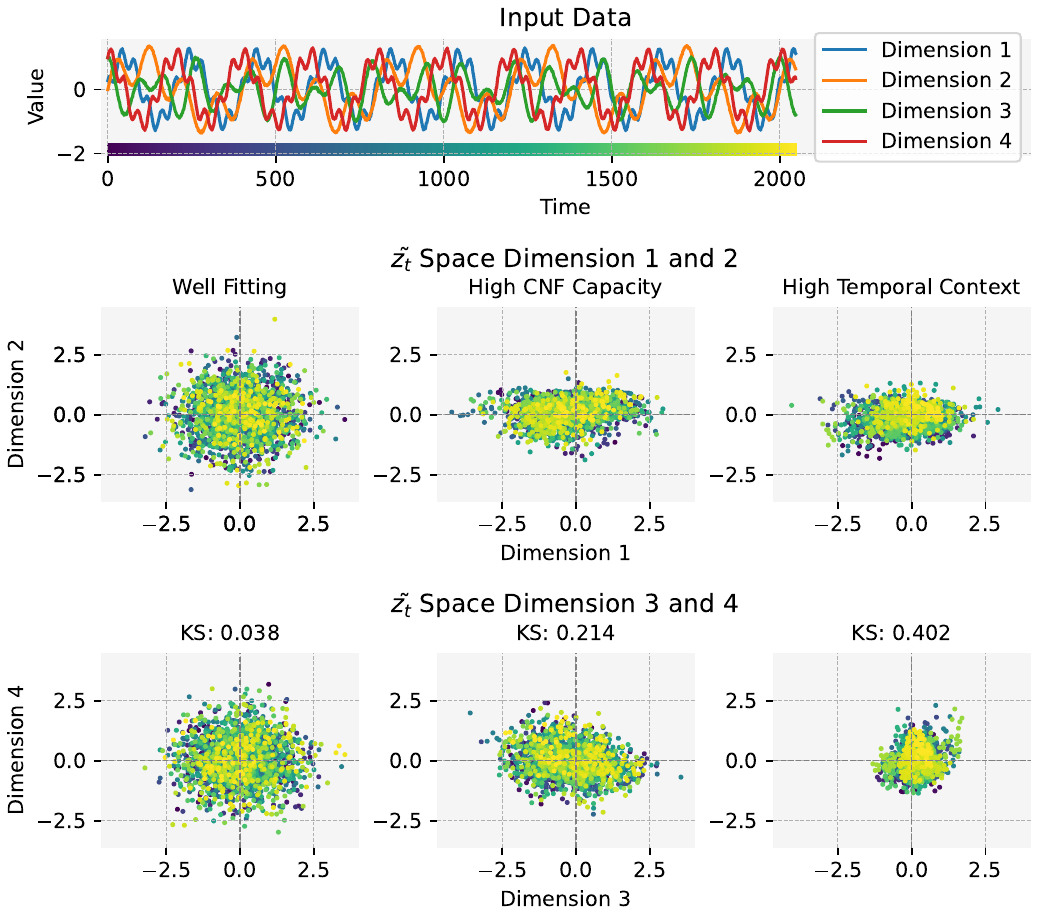}
    \caption{
    Example of 3 different CNF configurations and their impact on the $\wvzt$ latent space. 
    The first column shows a well-fitting and trustworthy model.
    The remaining columns show two models with too many layers or too much temporal context.
    A version including the $\mut$ space is in \cref{axssec:wrong_capacty}.
    }
    \label{fig:fit_vs_nofit_4d}
\end{figure}

\paragraph{Training compliance for unsupervised VS. oracle-based thresholding.}
We assess the value of the above-mentioned training compliance diagnostic
with real-world data (results reported in \cref{axtab:u_parametric_results}).
When compliance to inductive bias is met (see \textit{FIT} column in table),
the label-free, unsupervised MV-KS critical value-based AD 
performs close to the oracle-based, thresholded KS score in the Affiliation F1 metric.
When anomaly labels are available,
the oracle-thresholded NLL score is a strong competitor on the VUS-PR metric.
We acknowledge that KS performs worse than the NLL version on the tested real-world sequences.
We hypothesize that this might be due to the type of anomalies in these datasets,
since NLL scores can have shortcomings as shown in the synthetic experiments.
Additionally, we emphasize that we advocate for the MV-KS score for unsupervised AD,
which does not require labels or thresholding.

\paragraph{AD in real-world benchmark data.}
We report AD performance across a variety of metrics and methods in
univariate (\cref{tab:u_vus-pr_aggregated_comparison}) and multivariate (\cref{tab:m_vus-pr_aggregated_comparison}) time-series from the TSB-AD~\cite{liu2024} benchmark;
with a comparison to its leaderboard in~\cref{tab:tsb-ad_vus-pr_leaderboard}.
Overall, we achieve competitive results,
within the top-10 in the leaderboard when using NLL-scores.
Results in \Cref{tab:tsb-ad_vus-pr_leaderboard} are averaged over all the $1080$ sequences in the TSB-AD benchmark, for a diverse set of sequence types, many with difficult time-series and AD behavior (\eg binary and non-smooth data) for our CNF + LG-LDM system.

On a subset of datasets where per-sequence AD results are available,
we achieve competitive (better, or in par) results than baselines,
even when model learning is not fully compliant with the LG-LDM inductive bias (\cref{tab:u_vus-pr_aggregated_comparison,tab:m_vus-pr_aggregated_comparison}).
When compliance is successful,
\eg on 88\% of sequences in NEK and 100\% in the Stock datasets,
the NLL approach performs in the top-2 (mean VUS-PR: 54),
with MV-KS following close (mean VUS-PR: 44).
In the multivariate setting reported in~\cref{tab:m_vus-pr_aggregated_comparison},
both NLL (mean VUS-PR: 21) and MV-KS (mean VUS-PR: 18) are competitive.
On the contrary, 
the imposed latent dynamics are not compliant after training with the MITDB dataset
and, as predicted, both scoring rules underperform.

\begin{table*}[!ht]
\centering
    \caption{
    Aggregated univariate sequence results on the VUS-PR metric with mean and standard deviation results.
    Boldface numbers indicate best performing model within each dataset, with second-best underlined.
    Comparative results from TSB-AD are included, and a complete table is in the \cref{axssec:univarite_sequence_results}. 
    }
    \label{tab:u_vus-pr_aggregated_comparison}
\resizebox{\textwidth}{!}{
\begin{tabular}{lrrrrrr|rrr}
& & & & & & & \multicolumn{3}{c}{CNF \& LG-LDM} \\
Set & Sub-PCA & CNN & IForest & FITS & AutoEncoder & TimesNet & NLL & KS & FIT \\
\midrule
IOPS & 22.7 ± 20 & 26.0 ± 20 & \underline{27.7 ± 15} & 17.3 ± 15 & 25.5 ± 22 & 22.4 ± 21 & \textbf{33.5 ± 24} & 18.7 ± 13 & 13\% \\
MITDB & \textbf{36.3 ± 32} & 15.3 ± 13 & 9.8 ± ~~6 & 9.0 ± ~~7 & 7.4 ± ~~4 & 8.4 ± ~~5 & \underline{19.3 ± 13} & 16.3 ± ~~9 & 0\% \\
NEK & \textbf{90.5 ± 14} & 73.5 ± 22 & 59.3 ± 14 & 49.4 ± 25 & 50.6 ± 34 & 37.4 ± 26 & \underline{73.7 ± 25} & 70.0 ± 34 & 88\% \\
Stock & 84.0 ± 15 & 92.2 ± 17 & \textbf{99.3 ± ~~1} & 76.3 ± 28 & 72.0 ± 28 & 78.8 ± 27 & \underline{93.9 ± 9} & 81.9 ± 17 & 100\% \\
TODS & 53.9 ± 24 & 54.3 ± 26 & 51.8 ± 30 & 58.1 ± 19 & \textbf{65.2 ± 11} & 58.6 ± 20 & \underline{60.3 ± 25} & 48.5 ± 27 & 85\% \\
\end{tabular}
} 
\end{table*}

\begin{table*}[!ht]
    \centering
    \caption{
    Aggregated multivariate sequences results on the VUS-PR metric with mean and standard deviation results.
    Boldface numbers indicate best performing model within each dataset, with second-best underlined.
    Comparative results from TSB-AD are included, and a complete table is in the \cref{axssec:multivariate_sequence_results}.
    }
    \label{tab:m_vus-pr_aggregated_comparison}
\resizebox{\textwidth}{!}{
\begin{tabular}{lrrrrr|rrr}
& & & & & & \multicolumn{3}{c}{CNF \& LG-LDM} \\
Set & PCA & CNN & FITS & TimesNet & OmniAnomaly & NLL & KS & FIT \\
\midrule
LTDB & 24.4 ± 18 & 32.8 ± 25 & 22.8 ± 18 & 27.2 ± 23 & \textbf{44.4 ± 38} & \underline{35.0 ± 20} & 33.9 ± 21 & 25\% \\
MITDB & 6.5 ± ~~6 & \underline{14.0 ± 16} & 5.3 ± ~~6 & 6.8 ± ~~7 & 11.5 ± 11 & \textbf{15.2 ± 15} & 12.3 ± ~~9 & 18\% \\
\end{tabular}
} 
\end{table*}
\begin{table}[!ht]
    \centering
    \caption{
    Comparison to the TSB-AD leaderboard~\cite{liu2024} (as of 1st of May 2026), aggregated over all sequences on the VUS-PR metric. Ranking based on the average on the univariate sequences. Methods marked with * are pretrained models.
    }
    \label{tab:tsb-ad_vus-pr_leaderboard}
\begin{tabular}{lrr}
Rank \& Method & TSB-AD-U $\downarrow$ & TSB-AD-M \\
\midrule
1. TSPulse (FT)* & 0.55 & 0.39 \\
2. Time-RCD* & 0.52 & 0.32 \\
3. CHARM* & 0.50 & 0.39 \\
4. TSPulse (ZS)* & 0.48 & 0.36 \\ 
5. MMPAD & 0.44 & 0.35 \\
6. PCA & 0.42 & 0.31 \\
\textbf{7. CNF + LG-LDM (NLL)} & 0.41 & 0.35 \\
\textbf{17. CNF + LG-LDM (KS)} & 0.34 & 0.29 \\
\end{tabular}
\end{table}

\section{Discussion}
\label{sec:discussion}

The evaluation of the proposed framework
---with the specific implementation of a CNF with an LG-LDM
that constrains time-series representations to
temporally coherent whitened Gaussian trajectories---
demonstrates 
interpretability and robustness for unsupervised anomaly detection.
As shown in Figures~\ref{fig:main_figure}-\ref{fig:data_space_test},
checking inductive bias compliance (via GOF-tests) of incoming observations,
after mapping them to the latent space via the trained CNF,
provides an interpretable (visual) and robust (MV-KS) method
to identify \textit{unexpected behavior} for it to be flagged as an anomaly.

We emphasize that the proposed method enables AD
even in high-density areas of the trained model (see Figure~\ref{fig:main_figure})
due to our use of
($i$) inductive bias in the latent representations
---that impose time-evolving distributional constraints---
and ($ii$) statistical GOF tests over these latent trajectories.
That is, incoming observations might be mapped to
a clustered or structured latent region of the model,
which are often not directly detectable via NLL-based scoring~\citep{morningstar_density_2021,li_position_2025}.
However, even if they are located in a high-density area of our trained CNF,
it is by checking compliance with the prescribed inductive bias
(the LG-LDM time-varying dynamics in our experiments)
that we can successfully detect \textit{unexpected behavior}.

Dataset and anomaly heterogeneity dictates the performance trade-off between NLL scoring and GOF tests:
NLL perform better on isolated, spike or point anomalies (as in many real-data experiments),
while KS-based tests excel at capturing structural anomalies,
\eg frequency, amplitude, and other regime changes.
The former are specially well-captured in synthetic experiments,
highlighting the benefit of using the proposed MV-KS GOF test.

We acknowledge that highly erratic or flat sequences remain a challenge
for our evaluated CNF + LG-LDM framework,
as these sequence behaviors fundamentally complicate the underlying CNF calibration.
In addition, 
because latent dynamic compliance relies on statistical GOF testing
against the prescribed inductive bias,
its efficacy is inherently dependent on the chosen test statistic and the available sample size.
This dependency is particularly sensitive in high-dimensional spaces,
where statistical tests frequently suffer from low power due to the curse of dimensionality~\citep{bing_high-dimensional_2025,chen_normality_2023,morningstar_density_2021}.
Hence, our approach may require large temporal windows that can smooth over and hide isolated point anomalies. 

To mitigate this effect, we recommend scaling the MV-KS window size by the latent dimension $D$, with a recommended minimum of $\bigO{D^3}$.
A computationally cheaper alternative is to rely on subspace GOF testing.
If the imposed inductive bias allows it
(\eg factorized latent Gaussianity),
one may perform tests on lower-dimensional subspaces of the full latent representations. 
This reduces the sample complexity required for each test,
although it relies on multiple-testing,
and the limitations of compliance checking via marginal distributions.

A final limitation worth discussing is that the proposed framework is highly sensitive to model capacity and to specific dataset behaviors.
For the former, we recommend thorough hyper-parameter searches (see \cref{ssec:setup}).
For the latter, expert knowledge is highly valuable,
both for pre-processing of the input data
and for definition of the inductive bias:
\eg the LG-LDM is not a suitable inductive bias for all real-world datasets.
In any case,
we suggest relying on the monitoring of the training procedure
via GOF test-based inductive bias compliance as described in Section~\ref{ssec:unsupervised_AD}
and showcased empirically in Figure~\ref{fig:fit_vs_nofit_4d}.

\section{Conclusions}
\label{sec:conclusions}

We introduced a state-space, deep generative model with latent-space inductive bias for unsupervised time-series anomaly detection.
Specifically, we proposed a label-free anomaly detector
that scores latent dynamic compliance via GOF tests of representations learned
by a Conditional Normalizing Flow (CNF),
eliminating the need for manual threshold tuning. 

Crucially, our framework provides a practical model learning diagnostic
that builds upon the explicit inductive bias:
we apply a GOF test to the mapped training time-series
to verify whether the imposed inductive bias
is well-prescribed and realized by the learned model.
Consequently, the GOF critical-value decision rule becomes valid and trustworthy for AD deployment.
This principled decision rule remains effective even in regions of high observation likelihood.

Presented results support a broader, unifying perspective on deep generative anomaly detection, demonstrated by experiments on synthetic and real-world time-series. 
Rather than relying on model likelihood as a proxy for expected data behavior,
anomaly detection can be effectively framed as
testing whether learned representations satisfy an explicit,
application-motivated inductive bias over-time.

While the method shows good performance in identifying unexpected behavior without labels,
its limitations regarding adapting the model to the data and the multivariate GOF test's limitations suggest several avenues for further investigation. 
Future work will focus on learning other inductive bias over-time:
\eg learnable LG-LDM covariance matrix, non-Markovian LG-LDMs,
or non-linear models.
 Additional avenues for exploration are
 assessing the impact of MVN-KS window alignment,
 extending the framework to other data types,
 exploring its generative capabilities,
 and developing data-space explanations of unexpected behavior to complement our latent-space diagnostics.

\subsection*{Acknowledgements}
This work has been partly funded by the SFI NorwAI, (Centre for Research-based Innovation, 309834).
B. Baumgartner, thanks the financial support from the Research Council of Norway and the partners of the SFI NorwAI.

E. da Silva. and I. Urteaga thank the support of grant PID2023-146759NA-I00 funded by MICIU/AEI/10.13039/501100011033 and cofunded by the European Union,
as well as the  Basque Government through the BERC 2022-2025 program
and the Ministry of Science and Innovation: BCAM Severo Ochoa accreditation
CEX2021-001142-S/MICIN/AEI/10.13039/501100011033.
I. Urteaga is also partially supported by Grant RYC2023-045922-I funded by MICIU/AEI/10.13039/501100011033 and by ESF+.

E. da Silva’s work is supported by Portuguese national funds through
FCT – Foundation for Science and Technology, I.P.,
within the scope of the research unit
UID/00326 - Centre for Informatics and Systems of the University of Coimbra.

We sincerely appreciate the valuable comments and suggestions from Melissa Yan and Heri Ramampiaro, which helped us improve the quality of the manuscript.

\bibliography{main}

\clearpage
\appendix
\onecolumn

\section{Parametric results for synthetic case}
\label{axsec:synthetic_f1_res_table}

\begin{table}[h]
    \centering
    \caption{Various parametric F1 metric scores based on NLL and MVN-KS w/o critical value (CV) thresholding or the AUC-ROC curve as threshold on different window sizes for the test case in~\cref{fig:main_figure}.
    \Cref{tab:synthetic_res_table} contains the non-parametric scores for NLL and MVN-KS.
    }
    \label{axtab:synthetic_f1_res_table}
\begin{tabular}{lrrrrr}
Score Source & Standard F1 & Point-adjusted F1 & Event-based F1 & Range-based F1 & Affiliation F1 \\
\midrule
NLL & 68.9 & 98.0 & 95.3 & 36.0 & 85.0 \\
KS w=20 & 70.9 & \textbf{100.0} & \textbf{100.0} & 62.2 & 88.1 \\
KS w=20 (CV) & 55.0 & 97.1 & 93.2 & 46.7 & 86.9 \\
KS w=32 & 71.0 & 99.2 & 96.7 & 57.1 & 89.2 \\
KS w=32 (CV) & 57.7 & 95.7 & 90.8 & 48.7 & 87.2 \\
KS w=40 & 72.5 & 98.5 & 95.2 & 51.4 & 92.4 \\
KS w=40 (CV) & 60.1 & 94.9 & 89.8 & 43.7 & 87.7 \\
KS w=64 & 72.0 & 99.5 & 98.3 & 65.5 & \textbf{92.6} \\
KS w=64 (CV) & 66.0 & 95.4 & 91.8 & 63.3 & 91.9 \\
KS w=100 & \textbf{72.6} & 94.6 & 85.8 & 62.1 & 82.5 \\
KS w=100 (CV) & 59.3 & 90.4 & 83.0 & 56.2 & 80.9 \\
KS w=200 & 70.4 & 89.0 & 80.6 & \textbf{73.6} & 85.2 \\
KS w=200 (CV) & 60.1 & 88.2 & 80.2 & 70.6 & 81.8 \\
KS w=400 & 68.8 & 82.1 & 74.1 & 73.2 & 82.0 \\
KS w=400 (CV) & 60.8 & 60.8 & 60.1 & 35.5 & 64.9 \\
\end{tabular}
\end{table}

\clearpage

\section{Results per sequence and aggregated}
\label{axsec:results_per_sequence}

\subsection{Additional Univariate Sequences Results}
\label{axssec:univarite_sequence_results}

\begin{table}[h]
    \centering
    \caption{
    Per-sequence univariate sequence results for the VUS-PR metric. 
    It also includes in the FIT column whether the selected model meets the trustworthiness criteria.
    }
    \label{axtab:u_vus-pr_full}
    \footnotesize 
\begin{tabular}{lrrrrrr|rrr}
& & & & & & & \multicolumn{3}{c}{CNF \& LG-LDM} \\
Sequence & Sub-PCA & CNN & IForest & FITS & AutoEncoder & TimesNet & NLL & KS & FIT \\
\midrule
Stock, Id: 1 & 86.07 & \textbf{99.83} & \underline{99.5} & 82.19 & 74.11 & 85.54 & 96.95 & 89.64 & True \\
Stock, Id: 3 & 85.97 & 91.85 & \textbf{99.5} & 85.38 & 65.43 & 85.97 & \underline{94.29} & 85.12 & True \\
Stock, Id: 7 & 92.57 & 96.34 & \underline{99.7} & 91.32 & 85.91 & 93.02 & \textbf{99.72} & 91.56 & True \\
Stock, Id: 8 & 50.48 & 50.47 & \textbf{98.3} & 13.63 & 12.41 & 15.62 & \underline{71.19} & 45.12 & True \\
Stock, Id: 12 & 75.48 & \textbf{99.57} & \underline{98.99} & 61.46 & 57.58 & 72.68 & 93.35 & 70.96 & True \\
Stock, Id: 14 & 94.99 & \underline{99.73} & \textbf{99.84} & 94.07 & 94.96 & 94.26 & 98.99 & 93.71 & True \\
Stock, Id: 17 & 97.29 & \textbf{99.93} & 99.82 & 96.97 & 97.46 & 97.10 & \underline{99.9} & 97.53 & True \\
Stock, Id: 18 & 88.96 & \textbf{99.92} & \underline{98.77} & 85.06 & 87.99 & 86.08 & 96.44 & 81.92 & True \\
MITDB, Id: 1 & 8.85 & 6.57 & 6.74 & 6.24 & 5.35 & 6.70 & \underline{9.12} & \textbf{10.96} & False \\
MITDB, Id: 2 & \textbf{72.87} & 39.63 & 15.25 & 9.94 & 8.64 & 11.65 & \underline{41.32} & 24.29 & False \\
MITDB, Id: 3 & 1.84 & 7.66 & 2.55 & 2.40 & 3.38 & 2.39 & \textbf{15.17} & \underline{12.54} & False \\
MITDB, Id: 4 & \textbf{19.66} & \underline{17.82} & 16.60 & 11.43 & 10.45 & 11.88 & 16.41 & 17.05 & False \\
MITDB, Id: 6 & \textbf{25.14} & 3.79 & 3.89 & 3.63 & 3.68 & 3.45 & \underline{5.98} & 5.63 & False \\
MITDB, Id: 7 & \textbf{86.16} & 23.64 & 16.48 & 23.00 & 15.23 & 15.97 & \underline{32.79} & 31.26 & False \\
MITDB, Id: 8 & \textbf{39.43} & 7.76 & 7.04 & 6.43 & 5.22 & 6.49 & \underline{14.31} & 12.45 & False \\
IOPS, Id: 1 & \underline{22.62} & 12.09 & 16.49 & 2.15 & 19.90 & 20.31 & 14.47 & \textbf{22.82} & False \\
IOPS, Id: 2 & 1.18 & \underline{15.62} & 10.76 & 8.07 & \textbf{40.27} & 12.61 & 8.66 & 2.81 & False \\
IOPS, Id: 3 & 25.23 & 15.45 & \underline{30.6} & 10.25 & 2.92 & 11.97 & \textbf{34.03} & 8.25 & False \\
IOPS, Id: 4 & 10.61 & \textbf{67.83} & 13.42 & 46.94 & \underline{58.14} & 41.61 & 35.45 & 10.01 & False \\
IOPS, Id: 5 & 4.43 & 5.27 & 3.11 & 7.60 & \textbf{77.97} & 4.09 & 5.27 & \underline{16.68} & False \\
IOPS, Id: 7 & \textbf{62.01} & 4.35 & \underline{51.42} & 2.70 & 4.78 & 5.75 & 37.58 & 15.88 & False \\
IOPS, Id: 8 & 20.65 & \underline{37.96} & \textbf{39.0} & 23.12 & 7.08 & 17.54 & 17.17 & 6.73 & False \\
IOPS, Id: 9 & 23.63 & 16.59 & \underline{31.36} & 17.92 & 13.44 & 15.94 & \textbf{52.08} & 20.37 & False \\
IOPS, Id: 10 & \textbf{72.2} & 38.45 & 43.96 & 4.24 & 3.05 & 9.32 & \underline{70.03} & 41.79 & True \\
IOPS, Id: 12 & 14.59 & 15.36 & \textbf{27.12} & 17.45 & \underline{18.63} & 14.33 & 8.65 & 10.91 & True \\
IOPS, Id: 13 & 15.77 & \underline{69.18} & 22.15 & 48.85 & 53.62 & \textbf{87.64} & 52.19 & 20.24 & False \\
IOPS, Id: 14 & 4.42 & \underline{23.3} & 9.59 & 7.26 & 13.55 & \textbf{33.09} & 5.51 & 2.08 & False \\
IOPS, Id: 15 & 21.17 & 32.11 & 34.96 & \underline{35.61} & 26.61 & 33.94 & \textbf{54.78} & 22.47 & False \\
IOPS, Id: 16 & 9.47 & 16.34 & \textbf{32.76} & 12.32 & 24.65 & 13.31 & 30.03 & \underline{31.36} & False \\
IOPS, Id: 17 & 32.82 & 20.15 & \underline{48.94} & 14.55 & 17.43 & 14.39 & \textbf{77.34} & 48.26 & False \\
NEK, Id: 1 & 88.42 & 93.78 & 66.13 & 92.04 & \underline{96.01} & \textbf{97.55} & 70.12 & 6.32 & False \\
NEK, Id: 2 & \textbf{89.11} & 34.11 & 31.82 & 27.03 & 6.73 & 22.99 & 32.69 & \underline{38.31} & True \\
NEK, Id: 3 & \textbf{99.26} & 86.19 & 64.54 & 29.60 & 96.42 & 34.13 & \underline{98.73} & 90.44 & True \\
NEK, Id: 5 & \textbf{99.56} & 91.16 & 70.27 & 54.63 & 64.30 & 41.77 & 62.23 & \underline{98.17} & True \\
NEK, Id: 6 & \textbf{57.77} & 46.52 & 42.18 & 16.13 & 9.38 & 16.50 & 44.07 & \underline{50.94} & True \\
NEK, Id: 7 & \textbf{97.52} & 78.00 & 62.29 & 61.72 & 47.72 & 31.07 & \underline{94.97} & 90.38 & True \\
NEK, Id: 8 & \underline{96.9} & 74.38 & 71.07 & 66.65 & 36.53 & 32.34 & 91.88 & \textbf{97.72} & True \\
NEK, Id: 9 & \textbf{95.56} & 83.48 & 66.06 & 47.29 & 47.86 & 23.04 & \underline{94.62} & 87.58 & True \\
TODS, Id: 1 & 87.26 & \underline{89.49} & \textbf{96.67} & 82.35 & 87.57 & 82.54 & 88.15 & 88.09 & True \\
TODS, Id: 2 & 67.51 & \textbf{78.47} & 68.49 & 61.80 & 60.51 & 60.74 & \underline{71.35} & 27.28 & True \\
TODS, Id: 3 & 59.33 & 60.45 & 60.56 & 66.97 & 67.04 & 66.17 & \textbf{94.43} & \underline{85.01} & True \\
TODS, Id: 4 & 46.80 & 53.40 & 53.50 & \textbf{72.66} & 68.55 & 55.56 & \underline{69.22} & 36.46 & True \\
TODS, Id: 6 & 83.83 & 83.82 & \textbf{90.59} & 80.36 & 83.81 & 82.23 & \underline{89.01} & 86.91 & True \\
TODS, Id: 7 & 52.44 & 32.89 & 17.88 & 30.74 & 58.24 & 39.60 & \textbf{84.51} & \underline{75.74} & True \\
TODS, Id: 8 & 49.97 & 26.39 & 20.17 & 37.09 & \textbf{59.66} & 36.91 & \underline{59.08} & 50.10 & True \\
TODS, Id: 9 & 15.87 & 17.12 & 15.04 & 26.64 & \textbf{60.27} & 21.88 & 29.06 & \underline{41.48} & True \\
TODS, Id: 10 & 75.82 & \textbf{83.43} & \underline{81.95} & 72.85 & 70.64 & 76.84 & 27.34 & 13.99 & True \\
TODS, Id: 11 & 27.71 & 42.84 & 26.33 & 35.94 & 47.24 & \underline{65.87} & \textbf{67.18} & 47.15 & True \\
TODS, Id: 12 & 55.80 & 52.24 & 57.70 & 54.78 & \underline{67.23} & \textbf{67.94} & 43.93 & 19.81 & False \\
TODS, Id: 13 & 65.30 & 68.21 & \underline{70.95} & 65.95 & 64.75 & \textbf{72.67} & 30.95 & 19.67 & False \\
TODS, Id: 15 & 13.69 & 16.57 & 13.75 & \textbf{67.65} & \underline{52.13} & 32.98 & 29.22 & 38.97 & True \\
\end{tabular}
\end{table}

\begin{table}[h]
    \centering
    \caption{Aggregated results on non-parametric metrics over the chosen univariate sequences with the proposed CNF \& LG-LDM method. 
    }
    \label{axtab:u_non_parametic_results}
\begin{tabular}{lr|rrrrrrrrrrrrrrrrrrrrrrrrrrr}
& & \multicolumn{2}{c}{AUC-PR} & \multicolumn{2}{c}{AUC-ROC} & \multicolumn{2}{c}{VUS-PR} & \multicolumn{2}{c}{VUS-ROC} \\
Set & FIT & NLL & KS & NLL & KS & NLL & KS & NLL & KS \\
\midrule
IOPS & 13\% & 36 ± 22 & 19 ± 21 & 81 ± 12 & 74 ± 16 & 34 ± 24 & 15 ± 12 & 92 ± ~~8 & 77 ± 11 \\
MITDB & 0\% & 19 ± 14 & 15 ± ~~8 & 69 ± ~~7 & 71 ± ~~7 & 19 ± 13 & 16 ± ~~9 & 73 ± ~~7 & 73 ± ~~7 \\
NEK & 88\% & 61 ± 29 & 63 ± 36 & 89 ± 12 & 84 ± 17 & 74 ± 25 & 74 ± 29 & 96 ± ~~5 & 91 ± 10 \\
Stock & 100\% & 83 ± 14 & 8 ± ~~5 & 98 ± ~~1 & 51 ± ~~1 & 94 ± ~~9 & 81 ± 17 & 99 ± ~~1 & 92 ± ~~7 \\
TODS & 85\% & 39 ± 22 & 17 ± 15 & 77 ± 12 & 64 ± 11 & 60 ± 25 & 48 ± 27 & 89 ± ~~9 & 75 ± 16 \\
\end{tabular}
\end{table}

\begin{table}[h]
    \centering
    \caption{Aggregated results on parametric metrics over the chosen univariate sequences, except for the CV, where the critical value is used as the threshold with the proposed CNF \& LG-LDM method. 
    }
    \label{axtab:u_parametric_results}
\resizebox{\textwidth}{!}{
\begin{tabular}{l|rr|rr|rr|rr|rrr}
& \multicolumn{2}{c|}{Standard F1} & \multicolumn{2}{c|}{Point-adjusted F1} & \multicolumn{2}{c|}{Event-based F1} & \multicolumn{2}{c|}{Range-based F1} & \multicolumn{3}{c}{Affiliation F1} \\
Set &  NLL & KS & NLL & KS & NLL & KS & NLL & KS & NLL & KS & CV \\
\midrule
IOPS & 42 ± 20 & 25 ± 25 & 80 ± 22 & 46 ± 39 & 72 ± 23 & 19 ± 21 & 38 ± 13 & 14 ± ~~9 & 84 ± 12 & 75 ± ~~9 & 69 ± ~~2 \\
MITDB & 25 ± 13 & 23 ± 10 & 88 ± 14 & 80 ± 15 & 56 ± 32 & 45 ± 20 & 23 ± ~~8 & 23 ± ~~7 & 86 ± 10 & 84 ± ~~6 & 71 ± ~~2 \\
NEK & 65 ± 24 & 67 ± 30 & 92 ± 12 & 80 ± 26 & 86 ± 19 & 71 ± 35 & 45 ± 16 & 68 ± 20 & 95 ± 9 & 90 ± 12 & 79 ± 10 \\
Stock & 75 ± 15 & 15 ± ~~8 & 76 ± 15 & 15 ± ~~8 & 75 ± 15 & 15 ± ~~8 & 77 ± 12 & 38 ± 18 & 89 ± ~~6 & 68 ± ~~1 & 66 ± ~~1 \\
TODS & 44 ± 19 & 26 ± 18 & 64 ± 20 & 39 ± 28 & 50 ± 22 & 20 ± 10 & 56 ± 21 & 34 ± 18 & 77 ± ~~7 & 69 ± ~~3 & 63 ± ~~8 \\
\end{tabular}
} 
\end{table}

\clearpage

\subsection{Additional Multivariate Sequences Results}
\label{axssec:multivariate_sequence_results}

\begin{table}[h]
    \centering
    \caption{
    Per-sequence multivariate sequence results for the VUS-PR metric. 
    It also includes in the FIT column whether the selected model meets the trustworthiness criteria.
    }
    \label{axtab:m_vus-pr_full}
\begin{tabular}{lrrrrr|rrr}
& & & & & & \multicolumn{3}{c}{CNF \& LG-LDM} \\
Sequence & PCA & CNN & FITS & TimesNet & OmniAnomaly & NLL & KS & FIT \\
\midrule
MITDB, Id: 1 & 17.62 & \textbf{27.78} & 19.17 & 20.00 & 21.97 & \underline{25.66} & 20.47 & False \\
MITDB, Id: 2 & 0.41 & \underline{1.32} & 0.25 & 0.50 & \textbf{1.7} & 1.12 & 0.90 & False \\
MITDB, Id: 3 & 2.04 & 6.66 & 2.55 & 1.58 & 1.41 & \textbf{32.49} & \underline{21.28} & False \\
MITDB, Id: 4 & 4.12 & \textbf{8.35} & 4.55 & 5.48 & 4.75 & \underline{7.08} & 6.89 & False \\
MITDB, Id: 6 & 1.20 & 3.05 & 0.99 & 1.57 & \underline{10.63} & 3.85 & \textbf{19.66} & False \\
MITDB, Id: 7 & 13.90 & \textbf{55.78} & 9.83 & 20.92 & 32.50 & \underline{48.64} & 26.29 & False \\
MITDB, Id: 8 & 10.15 & \underline{23.3} & 8.95 & 7.50 & \textbf{24.24} & 18.62 & 17.93 & False \\
MITDB, Id: 9 & 1.08 & 2.15 & 0.98 & 1.44 & 1.26 & \textbf{7.14} & \underline{5.99} & True \\
MITDB, Id: 11 & 11.87 & \underline{14.68} & 2.26 & 5.24 & \textbf{18.73} & 13.40 & 3.69 & False \\
MITDB, Id: 12 & 6.60 & \underline{8.2} & 6.43 & 7.46 & 6.10 & 7.04 & \textbf{9.04} & True \\
MITDB, Id: 13 & 2.80 & \textbf{3.27} & 2.80 & 3.00 & 2.79 & 2.59 & \underline{3.16} & False \\
LTDB, Id: 1 & 31.84 & \underline{60.82} & 33.76 & 40.35 & \textbf{84.75} & 58.98 & 54.64 & False \\
LTDB, Id: 2 & 19.63 & 18.97 & 16.68 & 16.39 & 17.11 & \underline{33.53} & \textbf{36.34} & False \\
LTDB, Id: 3 & 44.80 & 46.49 & 40.24 & \underline{51.46} & \textbf{68.05} & 37.39 & 39.53 & True \\
LTDB, Id: 5 & 1.39 & 4.99 & 0.66 & 0.68 & \underline{7.61} & \textbf{10.04} & 5.10 & False \\
\end{tabular}
\end{table}

\begin{table}[h]
    \centering
    \caption{Aggregated results on non-parametric metrics over the chosen multivariate sequences with the proposed CNF \& LG-LDM method.
    }
    \label{axtab:m_non-parametric_results}
\begin{tabular}{lr|rrrrrrrrrrrrrrrrrrrrrrrrrrr}
& & \multicolumn{2}{c}{AUC-PR} & \multicolumn{2}{c}{AUC-ROC} & \multicolumn{2}{c}{VUS-PR} & \multicolumn{2}{c}{VUS-ROC} \\
Set & FIT & NLL & KS & NLL & KS & NLL & KS & NLL & KS \\
\midrule
LTDB & 25\% & 25 ± 14 & 21 ± 13 & 62 ± 16 & 62 ± 15 & 35 ± 20 & 32 ± 19 & 76 ± 14 & 74 ± 14 \\
MITDB & 18\% & 19 ± 15 & 12 ± 9 & 69 ± 11 & 72 ± 12 & 15 ± 15 & 12 ± 8 & 72 ± 11 & 74 ± 12 \\
\end{tabular}
\end{table}

\begin{table}[h]
    \centering
    \caption{Aggregated results on parametric metrics over the chosen multivariate sequences, except for the CV, where the critical value is used as the threshold with the proposed CNF \& LG-LDM method.
    }
    \label{axtab:m_parametric_results}
\resizebox{\textwidth}{!}{
\begin{tabular}{l|rr|rr|rr|rr|rrr}
& \multicolumn{2}{c|}{Standard F1} & \multicolumn{2}{c|}{Point-adjusted F1} & \multicolumn{2}{c|}{Event-based F1} & \multicolumn{2}{c|}{Range-based F1} & \multicolumn{3}{c}{Affiliation F1} \\
Set &  NLL & KS & NLL & KS & NLL & KS & NLL & KS & NLL & KS & CV \\
\midrule
LTDB & 34 ± 12 & 34 ± 15 & 77 ± 32 & 76 ± 11 & 57 ± 27 & 48 ± 16 & 22 ± 11 & 42 ± 24 & 79 ± 8 & 77 ± 7 & 73 ± 4 \\
MITDB & 26 ± 16 & 21 ± 11 & 87 ± 26 & 71 ± 28 & 77 ± 32 & 38 ± 33 & 28 ± 12 & 23 ± 15 & 91 ± 11 & 86 ± 9 & 73 ± 6 \\
\end{tabular}
} 
\end{table}

\clearpage

\subsection{Extended framework capacity issue visualization}
\label{axssec:wrong_capacty}

\begin{figure}[h!]
    \centering
    \includegraphics[width=0.8\linewidth]{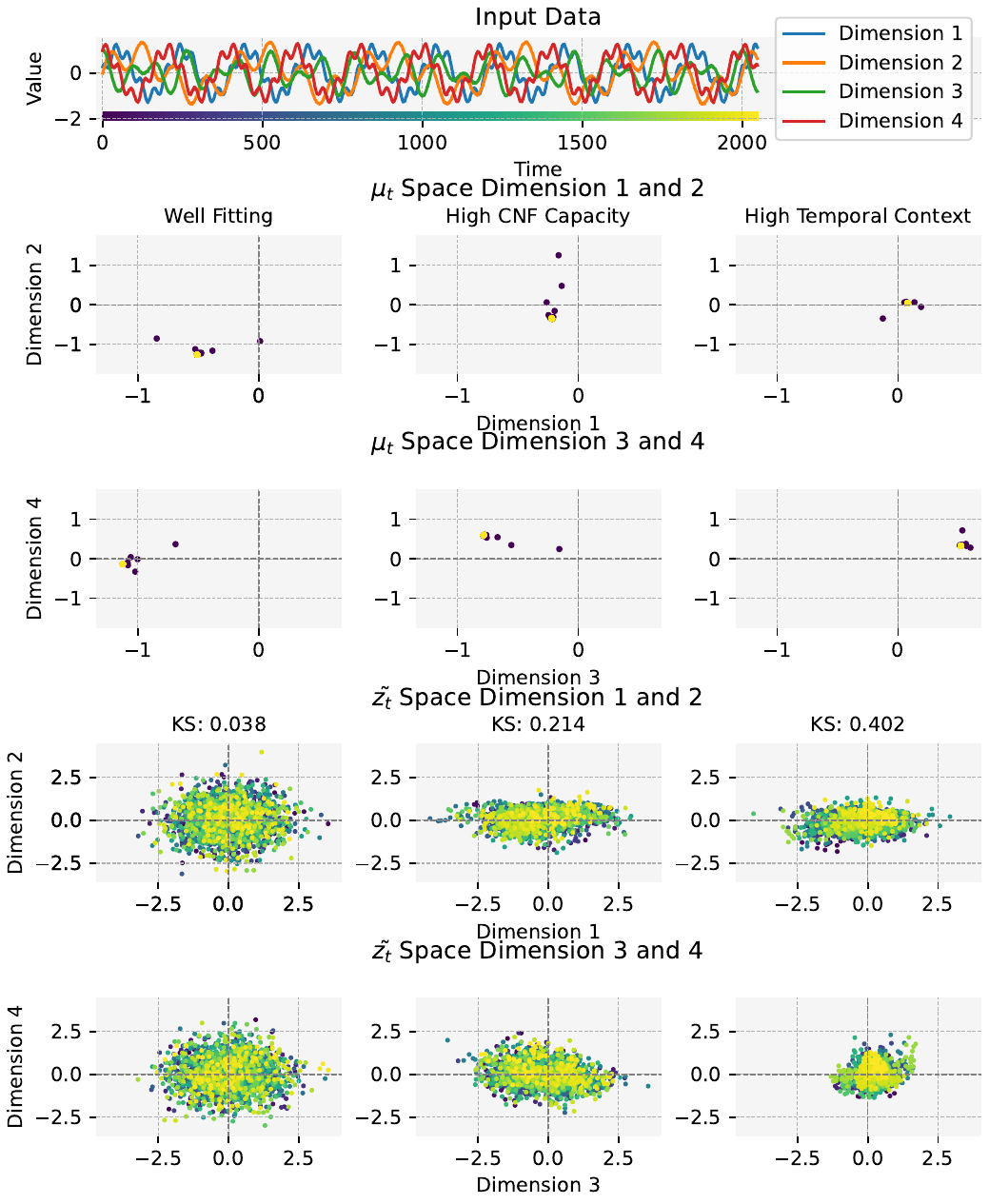}
    \caption{
    Example of 3 different CNF configurations and their impact on the $\wvzt$ latent space. 
    Row one contains the training sequence.
    Rows 2 and 3 contain the LG-LDM across all 4 dimensions.
    Rows 4 and 5 contain the $\wvzt$ space for all 4 dimensions.
    The first column shows an expected fit of a Gaussian behavior. 
    The other two columns contain CNF results with too many layers or too much temporal context.
    Row 4 includes in the title each model's MVN-KS value.
    }
    \label{axfig:fit_vs_nofit_4d}
\end{figure}

\clearpage
\section{Full TSB-AD Tables}

\begin{table}[t]
\centering
\caption{Grouped Table with comparison to TSB-AD: Univariate + VUS-PR}

\footnotesize 
\begin{tabular}{lrrrrrr|rrr}
Set & Sub-PCA & CNN & IForest & FITS & AutoEncoder & TimesNet & CNF-NLL & CNF-KS-V & FIT \\
\midrule
CATSv2 & \underline{26.5} & \textbf{31.6} & 7.9 & 17.1 & 17.8 & 10.4 & 15.9 & 20.1 & 0\% \\
Daphnet & \underline{42.2} & 39.5 & 36.4 & \textbf{42.7} & 9.5 & 39.0 & 26.6 & 17.4 & 0\% \\
Exathlon & \textbf{92.7 ± 10} & 60.5 ± 21 & 66.7 ± 13 & 54.6 ± 24 & 36.0 ± 35 & 52.9 ± 21 & 84.7 ± 17 & \underline{90.2 ± 16} & 0\% \\
IOPS & 22.7 ± 20 & \underline{26.0 ± 20} & 27.7 ± 15 & 17.3 ± 15 & 25.5 ± 22 & 22.4 ± 21 & \textbf{35.6 ± 24} & 18.1 ± 12 & 20\% \\
LTDB & \underline{55.6 ± 33} & 41.7 ± 15 & 33.8 ± 12 & 33.6 ± 16 & \textbf{69.2 ± 27} & 29.4 ± 13 & 49.3 ± 14 & 47.0 ± 10 & 62\% \\
MGAB & 0.5 ± 0 & 0.6 ± 0 & 0.4 ± 0 & 0.4 ± 0 & 0.6 ± 0 & 0.4 ± 0 & \textbf{4.5 ± 8} & \underline{4.0 ± 6} & 12\% \\
MITDB & \textbf{36.3 ± 32} & 15.3 ± 13 & 9.8 ± 6 & 9.0 ± 7 & 7.4 ± 4 & 8.4 ± 5 & \underline{19.1 ± 15} & 16.7 ± 12 & 14\% \\
MSL & \textbf{43.4 ± 41} & 32.3 ± 25 & 32.3 ± 19 & \underline{38.0 ± 29} & 18.4 ± 12 & 31.1 ± 17 & 25.2 ± 16 & 19.9 ± 16 & 50\% \\
NAB & \textbf{43.6 ± 30} & 19.1 ± 11 & 22.3 ± 12 & 24.4 ± 15 & 31.7 ± 19 & 20.3 ± 11 & 30.2 ± 18 & \underline{37.9 ± 22} & 61\% \\
NEK & \textbf{90.8 ± 15} & 70.5 ± 22 & 58.3 ± 15 & 43.3 ± 19 & 44.1 ± 31 & 28.8 ± 8 & 72.4 ± 27 & \underline{75.0 ± 30} & 100\% \\
OPPORTUNITY & \textbf{91.4 ± 21} & 40.2 ± 37 & 43.4 ± 19 & 7.1 ± 5 & 11.7 ± 12 & 5.3 ± 3 & 55.3 ± 37 & \underline{57.8 ± 31} & 19\% \\
Power & 7.6 & 7.8 & 8.2 & 7.2 & 8.8 & 7.5 & \underline{15.6} & \textbf{20.0} & 0\% \\
SED & 3.1 ± 1 & 6.2 ± 1 & \underline{35.8 ± 3} & 5.2 ± 2 & \textbf{41.1 ± 10} & 5.1 ± 1 & 7.6 ± 0 & 9.7 ± 2 & 0\% \\
SMAP & \underline{47.6 ± 38} & 34.4 ± 30 & 27.5 ± 30 & \textbf{47.8 ± 36} & 43.2 ± 41 & 36.8 ± 34 & 23.3 ± 27 & 31.0 ± 33 & 7\% \\
SMD & 46.2 ± 33 & \textbf{56.1 ± 16} & 35.2 ± 28 & 54.9 ± 16 & 14.1 ± 16 & \underline{55.2 ± 20} & 45.3 ± 22 & 18.0 ± 21 & 0\% \\
SVDB & \textbf{52.3 ± 25} & 22.1 ± 14 & 9.4 ± 7 & 10.1 ± 10 & 32.4 ± 19 & 9.7 ± 8 & 29.1 ± 16 & \underline{25.5 ± 14} & 6\% \\
SWaT & 39.3 & \textbf{68.2} & \underline{49.7} & 10.1 & 37.8 & 10.9 & 40.3 & 27.8 & 0\% \\
Stock & 84.0 ± 15 & 92.2 ± 17 & \textbf{99.3 ± 1} & 76.3 ± 28 & 72.0 ± 28 & 78.8 ± 27 & \underline{93.9 ± 8} & 82.2 ± 18 & 88\% \\
TAO & 93.0 ± 3 & \textbf{100.0 ± 0} & \underline{98.5 ± 0} & 91.1 ± 4 & 93.0 ± 3 & 90.9 ± 4 & 95.1 ± 5 & 91.2 ± 3 & 50\% \\
TODS & 53.9 ± 24 & 54.3 ± 26 & 51.8 ± 30 & 58.1 ± 19 & \textbf{65.2 ± 11} & 58.6 ± 20 & \underline{59.3 ± 28} & 48.7 ± 26 & 92\% \\
UCR & \underline{11.9 ± 24} & 4.8 ± 11 & 2.3 ± 4 & 2.0 ± 3 & 9.3 ± 17 & 2.3 ± 5 & 11.1 ± 17 & \textbf{13.1 ± 24} & 26\% \\
WSD & 9.4 ± 13 & \underline{24.5 ± 29} & 14.4 ± 20 & 14.5 ± 19 & 14.1 ± 22 & \textbf{26.6 ± 29} & 20.5 ± 23 & 5.8 ± 7 & 20\% \\
YAHOO & 13.7 ± 25 & \underline{53.0 ± 43} & 43.5 ± 43 & 18.5 ± 20 & 28.7 ± 33 & 29.0 ± 28 & \textbf{73.3 ± 40} & 26.1 ± 33 & 100\% \\
\end{tabular}
\end{table}

\begin{table}
\centering
\caption{Aggregated CNF Metric Table over all Results in Set: Univariate}

\scriptsize 
\begin{tabular}{lr|rr|rr|rr|rr}
Set & FIT & NL AUC-PR & KS AUC-PR & NL AUC-ROC & KS AUC-ROC & NL VUS-PR & KS VUS-PR & NL VUS-ROC & KS VUS-ROC \\
\midrule
CATSv2 & 0\% & 18 & 23 & 63 & 73 & 16 & 20 & 66 & 72 \\
Daphnet & 0\% & 23 & 15 & 89 & 83 & 27 & 17 & 91 & 83 \\
Exathlon & 0\% & 85 ± 17 & 89 ± 17 & 97 ± 5 & 98 ± 2 & 86 ± 17 & 91 ± 16 & 97 ± 5 & 99 ± 2 \\
IOPS & 18\% & 38 ± 25 & 16 ± 20 & 82 ± 14 & 68 ± 19 & 32 ± 25 & 14 ± 12 & 92 ± 8 & 73 ± 17 \\
LTDB & 67\% & 33 ± 14 & 30 ± 11 & 62 ± 14 & 63 ± 14 & 47 ± 14 & 45 ± 11 & 77 ± 11 & 76 ± 12 \\
MGAB & 11\% & 1 ± 2 & 0 ± 0 & 55 ± 7 & 57 ± 5 & 4 ± 8 & 3 ± 5 & 83 ± 5 & 83 ± 4 \\
MITDB & 14\% & 19 ± 15 & 15 ± 11 & 69 ± 8 & 71 ± 7 & 19 ± 15 & 16 ± 12 & 72 ± 9 & 74 ± 7 \\
MSL & 50\% & 20 ± 15 & 12 ± 10 & 67 ± 11 & 62 ± 11 & 25 ± 16 & 15 ± 10 & 74 ± 13 & 69 ± 11 \\
NAB & 57\% & 25 ± 16 & 32 ± 22 & 61 ± 14 & 69 ± 15 & 29 ± 17 & 36 ± 22 & 68 ± 14 & 73 ± 15 \\
NEK & 89\% & 56 ± 33 & 57 ± 37 & 83 ± 14 & 84 ± 13 & 67 ± 29 & 64 ± 33 & 93 ± 6 & 88 ± 11 \\
OPPORTUNITY & 24\% & 53 ± 37 & 55 ± 31 & 87 ± 20 & 87 ± 18 & 56 ± 38 & 56 ± 32 & 88 ± 19 & 88 ± 17 \\
Power & 0\% & 14 & 22 & 61 & 68 & 16 & 20 & 65 & 71 \\
SED & 0\% & 6 ± 5 & 7 ± 3 & 54 ± 15 & 60 ± 13 & 10 ± 5 & 12 ± 4 & 70 ± 12 & 74 ± 9 \\
SMAP & 6\% & 24 ± 28 & 25 ± 31 & 75 ± 10 & 72 ± 15 & 26 ± 28 & 29 ± 32 & 81 ± 11 & 80 ± 14 \\
SMD & 0\% & 47 ± 20 & 15 ± 20 & 89 ± 8 & 66 ± 18 & 44 ± 21 & 15 ± 20 & 93 ± 7 & 72 ± 17 \\
SVDB & 11\% & 24 ± 13 & 22 ± 17 & 78 ± 8 & 77 ± 10 & 28 ± 16 & 24 ± 14 & 85 ± 7 & 83 ± 9 \\
SWaT & 0\% & 74 & 71 & 83 & 81 & 40 & 27 & 69 & 56 \\
Stock & 95\% & 83 ± 9 & 10 ± 5 & 97 ± 2 & 51 ± 1 & 94 ± 7 & 79 ± 22 & 99 ± 1 & 91 ± 10 \\
TAO & 67\% & 58 ± 28 & 10 ± 4 & 90 ± 6 & 52 ± 1 & 95 ± 4 & 91 ± 2 & 99 ± 1 & 96 ± 1 \\
TODS & 93\% & 45 ± 25 & 22 ± 21 & 77 ± 13 & 66 ± 14 & 63 ± 28 & 52 ± 26 & 90 ± 9 & 78 ± 16 \\
UCR & 15\% & 11 ± 17 & 8 ± 16 & 72 ± 17 & 76 ± 17 & 11 ± 16 & 10 ± 17 & 81 ± 14 & 81 ± 14 \\
WSD & 15\% & 30 ± 23 & 4 ± 11 & 75 ± 15 & 62 ± 17 & 21 ± 22 & 5 ± 8 & 90 ± 10 & 72 ± 14 \\
YAHOO & 98\% & 60 ± 43 & 3 ± 14 & 88 ± 18 & 62 ± 15 & 61 ± 33 & 27 ± 23 & 94 ± 11 & 77 ± 14 \\
\end{tabular}
\end{table}

\begin{table}
\centering
\caption{Aggregated CNF Metric Table over all Results in Set: Univariate}

\scriptsize 
\begin{tabular}{lrrrrrrrrrr}
Set & NL-F1 & KS-F1 & NL PA-F1 & KS PA-F1 & NL Event-F1 & KS Event-F1 & NL R-F1 & KS R-F1 & NL Affi-F1 & KS Affi-F1 \\
\midrule
CATSv2 & 19 & 34 & 53 & 86 & 19 & 56 & 4 & 24 & 68 & 78 \\
Daphnet & 36 & 31 & 61 & 57 & 11 & 31 & 7 & 47 & 75 & 82 \\
Exathlon & 86 ± 13 & 92 ± 11 & 99 ± 3 & 98 ± 9 & 97 ± 8 & 96 ± 10 & 41 ± 17 & 90 ± 15 & 98 ± 5 & 99 ± 2 \\
IOPS & 43 ± 24 & 22 ± 25 & 75 ± 28 & 41 ± 40 & 70 ± 29 & 17 ± 20 & 40 ± 16 & 14 ± 10 & 86 ± 12 & 72 ± 7 \\
LTDB & 39 ± 10 & 40 ± 10 & 77 ± 27 & 74 ± 23 & 59 ± 30 & 53 ± 24 & 26 ± 8 & 44 ± 11 & 79 ± 9 & 77 ± 8 \\
MGAB & 5 ± 5 & 1 ± 1 & 35 ± 20 & 9 ± 7 & 19 ± 17 & 2 ± 2 & 9 ± 8 & 4 ± 5 & 76 ± 8 & 76 ± 8 \\
MITDB & 24 ± 14 & 23 ± 10 & 84 ± 19 & 80 ± 14 & 53 ± 27 & 43 ± 23 & 21 ± 9 & 26 ± 12 & 81 ± 10 & 84 ± 8 \\
MSL & 29 ± 16 & 21 ± 10 & 88 ± 22 & 47 ± 26 & 60 ± 39 & 25 ± 17 & 33 ± 14 & 37 ± 17 & 86 ± 13 & 77 ± 6 \\
NAB & 31 ± 15 & 41 ± 17 & 97 ± 15 & 88 ± 20 & 85 ± 26 & 68 ± 34 & 35 ± 11 & 54 ± 13 & 92 ± 9 & 87 ± 11 \\
NEK & 57 ± 29 & 61 ± 29 & 93 ± 10 & 75 ± 29 & 84 ± 22 & 64 ± 35 & 39 ± 11 & 71 ± 15 & 92 ± 11 & 89 ± 10 \\
OPPORTUNITY & 63 ± 32 & 64 ± 26 & 83 ± 28 & 85 ± 20 & 69 ± 36 & 71 ± 30 & 32 ± 23 & 58 ± 24 & 89 ± 12 & 88 ± 12 \\
Power & 22 & 26 & 95 & 87 & 35 & 49 & 21 & 35 & 85 & 91 \\
SED & 14 ± 9 & 14 ± 6 & 38 ± 28 & 49 ± 25 & 17 ± 15 & 27 ± 13 & 14 ± 6 & 19 ± 7 & 72 ± 4 & 73 ± 7 \\
SMAP & 30 ± 27 & 29 ± 31 & 79 ± 37 & 55 ± 43 & 65 ± 42 & 46 ± 47 & 30 ± 14 & 40 ± 30 & 94 ± 8 & 91 ± 11 \\
SMD & 52 ± 19 & 21 ± 23 & 98 ± 4 & 53 ± 31 & 91 ± 15 & 30 ± 34 & 39 ± 10 & 26 ± 17 & 96 ± 6 & 87 ± 9 \\
SVDB & 30 ± 11 & 30 ± 17 & 95 ± 6 & 80 ± 18 & 77 ± 22 & 54 ± 26 & 29 ± 9 & 28 ± 11 & 91 ± 6 & 86 ± 7 \\
SWaT & 79 & 77 & 86 & 52 & 43 & 31 & 24 & 17 & 69 & 70 \\
Stock & 76 ± 10 & 17 ± 8 & 77 ± 11 & 17 ± 8 & 75 ± 12 & 17 ± 8 & 76 ± 11 & 36 ± 13 & 85 ± 7 & 68 ± 1 \\
TAO & 57 ± 21 & 17 ± 7 & 38 ± 17 & 17 ± 7 & 36 ± 16 & 17 ± 7 & 34 ± 16 & 32 ± 10 & 68 ± 1 & 68 ± 1 \\
TODS & 49 ± 23 & 29 ± 21 & 66 ± 23 & 43 ± 32 & 57 ± 26 & 23 ± 22 & 57 ± 18 & 32 ± 17 & 79 ± 9 & 71 ± 8 \\
UCR & 17 ± 19 & 14 ± 18 & 76 ± 34 & 47 ± 39 & 49 ± 43 & 26 ± 36 & 23 ± 18 & 20 ± 19 & 89 ± 12 & 86 ± 11 \\
WSD & 38 ± 22 & 7 ± 15 & 93 ± 16 & 12 ± 20 & 90 ± 20 & 7 ± 13 & 48 ± 15 & 14 ± 11 & 96 ± 8 & 76 ± 8 \\
YAHOO & 63 ± 40 & 6 ± 15 & 64 ± 40 & 6 ± 16 & 64 ± 40 & 5 ± 13 & 66 ± 35 & 18 ± 17 & 89 ± 12 & 77 ± 7 \\
\end{tabular}
\end{table}

\begin{table}
\centering
\caption{Grouped Table with comparison to TSB-AD: Multivarite + VUS-PR}
\begin{tabular}{lrrrrr|rrr}
Set & PCA & CNN & FITS & TimesNet & OmniAnomaly & CNF-NLL & CNF-KS-V & FIT \\
\midrule
CATSv2 & \underline{13.1 ± 9} & 9.3 ± 5 & \textbf{14.0 ± 6} & 7.7 ± 3 & 4.5 ± 3 & 11.7 ± 14 & 5.4 ± 3 & 0\% \\
Daphnet & 13.0 & 20.9 & \underline{33.4} & 27.5 & \textbf{34.0} & 21.5 & 4.8 & 0\% \\
Exathlon & \textbf{94.5 ± 1} & 57.8 ± 41 & 41.7 ± 33 & 30.6 ± 21 & 58.0 ± 49 & \underline{90.8 ± 11} & 59.0 ± 38 & 0\% \\
GECCO & \textbf{20.1} & 3.3 & 3.3 & \underline{3.5} & 2.1 & \underline{3.5} & 3.3 & 0\% \\
GHL & 1.3 ± 1 & \underline{3.0 ± 4} & 0.7 ± 0 & 0.8 ± 0 & \textbf{5.6 ± 8} & 1.9 ± 2 & 1.8 ± 1 & 0\% \\
Genesis & 1.9 & \textbf{10.1} & \underline{9.6} & 1.9 & 0.3 & 2.4 & 2.0 & 0\% \\
LTDB & 24.4 ± 18 & 32.8 ± 25 & 22.8 ± 18 & 27.2 ± 23 & \textbf{44.4 ± 38} & \underline{41.4 ± 31} & 40.9 ± 31 & 75\% \\
MITDB & 6.5 ± 6 & \underline{14.0 ± 16} & 5.3 ± 6 & 6.8 ± 7 & 11.5 ± 11 & \textbf{17.3 ± 16} & 13.6 ± 9 & 27\% \\
SMAP & 18.9 ± 27 & \textbf{50.8 ± 29} & 22.8 ± 12 & 22.3 ± 11 & 29.7 ± 21 & \underline{48.0 ± 36} & 9.5 ± 14 & 0\% \\
SVDB & 11.2 ± 9 & 18.6 ± 14 & 10.4 ± 9 & 11.1 ± 9 & \textbf{35.2 ± 23} & 22.2 ± 15 & \underline{24.8 ± 19} & 89\% \\
TAO & \textbf{100.0 ± 0} & \underline{99.8 ± 0} & 78.1 ± 17 & 78.6 ± 16 & 81.3 ± 16 & 81.8 ± 21 & 70.4 ± 26 & 18\% \\

\end{tabular}
\end{table}

\begin{table}
\centering
\caption{Aggregated CNF Metric Table over all Results in Set: Multivariate}

\scriptsize 
\begin{tabular}{lr|rr|rr|rr|rr}
Set & FIT & NL AUC-PR & KS AUC-PR & NL AUC-ROC & KS AUC-ROC & NL VUS-PR & KS VUS-PR & NL VUS-ROC & KS VUS-ROC \\
\midrule
CATSv2 & 0\% & 13 ± 15 & 6 ± 3 & 61 ± 13 & 59 ± 9 & 11 ± 12 & 6 ± 3 & 59 ± 15 & 59 ± 8 \\
Daphnet & 0\% & 18 & 4 & 83 & 24 & 21 & 5 & 86 & 24 \\
Exathlon & 0\% & 90 ± 12 & 58 ± 37 & 99 ± 1 & 89 ± 11 & 91 ± 11 & 59 ± 38 & 99 ± 1 & 90 ± 10 \\
GECCO & 0\% & 2 & 1 & 52 & 49 & 3 & 3 & 63 & 58 \\
GHL & 0\% & 2 ± 1 & 2 ± 1 & 50 ± 24 & 61 ± 6 & 2 ± 1 & 2 ± 1 & 49 ± 24 & 60 ± 7 \\
Genesis & 0\% & 0 & 1 & 65 & 83 & 2 & 2 & 87 & 90 \\
LTDB & 80\% & 34 ± 25 & 35 ± 22 & 67 ± 14 & 68 ± 17 & 43 ± 27 & 43 ± 28 & 80 ± 12 & 79 ± 15 \\
MITDB & 23\% & 19 ± 15 & 12 ± 9 & 71 ± 12 & 73 ± 12 & 15 ± 15 & 12 ± 9 & 74 ± 11 & 75 ± 12 \\
SMAP & 0\% & 36 ± 33 & 2 ± 2 & 74 ± 12 & 47 ± 14 & 48 ± 36 & 6 ± 7 & 89 ± 9 & 58 ± 14 \\
SVDB & 84\% & 22 ± 15 & 22 ± 18 & 69 ± 10 & 73 ± 12 & 24 ± 16 & 25 ± 19 & 77 ± 8 & 80 ± 11 \\
TAO & 15\% & 43 ± 34 & 9 ± 4 & 78 ± 18 & 50 ± 1 & 84 ± 19 & 74 ± 25 & 95 ± 7 & 85 ± 15 \\
\end{tabular}
\end{table}

\begin{table}
\centering
\caption{Aggregated CNF Metric Latex Table for file list: Multivariate}

\scriptsize 
\begin{tabular}{lrrrrrrrrrr}
Set & NL F1 & KS F1 & NL PA-F1 & KS-V PA-F1 & NL Event-F1 & KS Event-F1 & NL R-F1 & KS R-F1 & NL Affi-F1 & KS Affi-F1 \\
\midrule
CATSv2 & 19 ± 16 & 11 ± 5 & 45 ± 22 & 11 ± 5 & 34 ± 25 & 10 ± 4 & 20 ± 11 & 40 ± 29 & 68 ± 2 & 69 ± 3 \\
Daphnet & 32 & 13 & 74 & 14 & 11 & 13 & 12 & 70 & 75 & 76 \\
Exathlon & 90 ± 8 & 66 ± 32 & 100 ± 1 & 75 ± 30 & 95 ± 9 & 70 ± 33 & 34 ± 4 & 66 ± 28 & 98 ± 3 & 92 ± 9 \\
GECCO & 5 & 3 & 23 & 3 & 14 & 3 & 6 & 20 & 67 & 67 \\
GHL & 5 ± 4 & 4 ± 2 & 9 ± 12 & 5 ± 2 & 3 ± 3 & 3 ± 1 & 2 ± 1 & 5 ± 3 & 77 ± 10 & 72 ± 6 \\
Genesis & 1 & 3 & 8 & 6 & 1 & 2 & 2 & 5 & 69 & 84 \\
LTDB & 39 ± 19 & 41 ± 16 & 83 ± 30 & 82 ± 25 & 69 ± 29 & 63 ± 30 & 35 ± 20 & 38 ± 8 & 84 ± 11 & 84 ± 11 \\
MITDB & 24 ± 16 & 22 ± 13 & 88 ± 24 & 72 ± 30 & 82 ± 25 & 44 ± 33 & 29 ± 11 & 24 ± 13 & 91 ± 13 & 87 ± 11 \\
SMAP & 41 ± 33 & 6 ± 6 & 100 ± 1 & 38 ± 34 & 88 ± 27 & 10 ± 10 & 47 ± 25 & 27 ± 14 & 99 ± 1 & 86 ± 15 \\
SVDB & 27 ± 13 & 28 ± 17 & 89 ± 16 & 79 ± 19 & 73 ± 24 & 52 ± 29 & 27 ± 9 & 33 ± 14 & 86 ± 11 & 83 ± 10 \\
TAO & 48 ± 30 & 16 ± 6 & 39 ± 33 & 16 ± 6 & 38 ± 32 & 16 ± 6 & 36 ± 35 & 46 ± 11 & 73 ± 11 & 68 ± 1 \\
\end{tabular}

\end{table}
\clearpage

\section{Validity of the MV-KS Test Under Nonlinear CNF Mapping}
\label{axsec:ks_validity}

\begin{proposition}[Distribution-Free MV-KS under CNF]
\label{prop:ks_validity}
Let $F_\theta: \mathbb{R}^D \times \mathbb{R}^M \to \mathbb{R}^D$ be the trained conditional normalizing flow, 
a diffeomorphism in its first argument for any condition $\mW_t$. Let the prescribed latent dynamics induce 
sufficient statistics $(\vmu_t, \mSigma_t)$ via $\vmu_t = \psi(\vmu_{t-1}; \phi)$. Define the \textbf{whitened residual}:
\begin{equation}
\wvzt = \mSigma_t^{-1/2}\bigl(F_\theta(\vx_t \mid \mW_t) - \vmu_t\bigr).
\end{equation}
Under the null hypothesis $\mathcal{H}_0$ that $\{\vx_t\}$ complies with the inductive bias 
(i.e., the model is well-specified and the data follow expected behavior):
\begin{enumerate}
    \item \textbf{Independence:} The random variables $\{\wvzt\}_{t=1}^T$ are mutually independent.
    \item \textbf{Distribution:} Each $\wvzt \sim \mathcal{N}(\mathbf{0}, \mathbf{I}^D)$.
    \item \textbf{Test Validity:} The multivariate Kolmogorov-Smirnov statistic computed on 
    $\{\wvzt\}_{t=1}^T$ against $\mathcal{N}(\mathbf{0}, \mathbf{I}^D)$ has a null distribution 
    that depends \emph{only} on the sample size $T$ and dimension $D$, invariant to $\theta$, $\phi$, 
    or the nonlinearity of $F_\theta$.
\end{enumerate}
\end{proposition}

\begin{proof}
By the change-of-variable formula of the CNF, the pushforward density of 
$\vz_t = F_\theta(\vx_t \mid \mW_t)$ under the model is exactly $\mathcal{N}(\vmu_t, \mSigma_t)$. 
The LG-LDM with the CNF $F_\theta$ places some of the temporal dependence into the sufficient statistics $(\vmu_t, \mSigma_t)$; 
there is no remaining temporal correlation in the innovation noise. Whitening $\vz_t$ by 
$\mSigma_t^{-1/2}(\vz_t - \vmu_t)$ therefore yields independent standard normal variables.

The MV-KS test is a test of a fully specified continuous distribution against i.i.d. samples. 
Its null distribution is universal (distribution-free) and invariant under the bijective 
transformation used to generate the samples. The nonlinearity of $F_\theta$ affects the 
\emph{power} of the test (ease of detecting violations) but not the \emph{validity} of 
Type-I error control.
\end{proof}

\end{document}